\documentclass{article} 
\usepackage{iclr2024_conference,times}
\iclrfinalcopy

\usepackage{amsmath,amsfonts,bm}

\def\eqref#1{equation~\ref{#1}}

\def\1{\bm{1}}

\DeclareMathAlphabet{\mathsfit}{\encodingdefault}{\sfdefault}{m}{sl}
\SetMathAlphabet{\mathsfit}{bold}{\encodingdefault}{\sfdefault}{bx}{n}

    \usepackage{amsthm}
\usepackage{graphicx}

\usepackage[T1]{fontenc}    
\usepackage{booktabs}       

\usepackage{amsfonts,amsmath,amssymb,bbm}
\usepackage{enumitem}
\usepackage{graphicx}
\usepackage{adjustbox}
\usepackage[UKenglish]{isodate}
\usepackage{graphicx}
\usepackage[font=small,labelfont=it]{caption}
\usepackage{subcaption}
\usepackage{hyperref} 
    \hypersetup{
    	colorlinks = true,
    	linkcolor = blue,
    	anchorcolor = blue,
    	citecolor = blue,
    	filecolor = blue,
    	urlcolor = blue
    }
\usepackage{cleveref}
\usepackage{float}
\usepackage{enumitem}
\usepackage{multirow}
\usepackage{fancyvrb}
\usepackage{rotating}
    \usepackage{tikz}
    \usepackage{tikz-cd} 
    \pgfdeclarelayer{edgelayer}
    \pgfdeclarelayer{nodelayer}
    \pgfsetlayers{edgelayer,nodelayer,main}
    \tikzstyle{new style 0}=[fill={rgb,255: red,255; green,94; blue,247}, draw=black, shape=circle]
    \tikzstyle{pointy}=[fill=white, draw=black, shape=circle]
    \tikzstyle{pointy}=[->]
\usepackage{fontawesome}

\usepackage{algpseudocode}
\usepackage[linesnumbered,lined,boxed,commentsnumbered,ruled,longend]{algorithm2e}

\SetCommentSty{mycommfont}

\usepackage{lscape}

\renewcommand{\phi}{\varphi}

\newcommand\sbullet[1][.5]{\mathbin{\vcenter{\hbox{\scalebox{#1}{$\bullet$}}}}}

\newcommand{\xxx}{\mathcal{X}}
\newcommand{\yyy}{\mathcal{Y}}

\newcommand{\eqdef}{\ensuremath{\stackrel{\mbox{\upshape\tiny def.}}{=}}}

\newtheorem{definition}{Definition}
\newtheorem{proposition}{Proposition}
\newtheorem{assumption}{Assumption}
\newtheorem{theorem}{Theorem}

\newtheorem{example}{Example}
\newtheorem{exampledets}{Example - Details}
\newtheorem{lemma}{Lemma}

\NewDocumentCommand{\anastasis}{mo}{
    \IfValueF{#2}{
                        {{\scriptsize
                            \textcolor{violet}{ 
                            \textbf{A:}
                            \textit{{#1}}
                            }
                        }}
        }
    \IfValueT{#2}{
                        \marginnote{{\scriptsize
                            \textcolor{violet}{ 
                            \textbf{A:}
                            \textit{{#1}}
                            }
                        }}
        }
                    }

\usepackage{appendix}
\usepackage{comment}

\newcommand{\first}[1]{\mathbf{\textcolor{red}{#1}}}
\newcommand{\second}[1]{\mathbf{\textcolor{blue}{#1}}}
\newcommand{\third}[1]{\mathbf{\textcolor{violet}{#1}}}

\definecolor{bgrn}{rgb}{0.0, 0.26, 0.15}

\usepackage{hyperref}
\usepackage{wrapfig}
\usepackage{float}
\usepackage{algpseudocode}
\usepackage[linesnumbered,lined,boxed,commentsnumbered,ruled,longend]{algorithm2e}

\usepackage{multicol}
\usepackage{url}
\usepackage{graphicx}
\usepackage{subcaption}
\usepackage{dsfont}

\usepackage{amsmath}

\title{Neural Snowflakes: Universal Latent Graph Inference via Trainable Latent Geometries}

\author{Haitz Sáez de Oc\'{a}riz Borde\thanks{Equal Contribution.} \\
University of Oxford\\
Oxford, UK\\
\texttt{chri6704@ox.ac.uk} \\
\And
Anastasis Kratsios$^*$ \\
Department of Mathematics\\
McMaster University and the Vector Institute \\
Ontario, Canada\\
\texttt{kratsioa@mcmaster.ca} \\
}

\iclrfinalcopy 

\begin{document}

\maketitle

\begin{abstract}
The inductive bias of a graph neural network (GNN) is largely encoded in its specified graph. Latent graph inference relies on latent geometric representations to dynamically rewire or infer a GNN's graph to maximize the GNN's predictive downstream performance, but it lacks solid theoretical foundations in terms of embedding-based representation guarantees. This paper addresses this issue by introducing a trainable deep learning architecture, coined \textit{neural snowflake}, that can adaptively implement fractal-like metrics on $\mathbb{R}^d$. We prove that any given finite weighted graph can be isometrically embedded by a standard MLP encoder, together with the metric implemented by the neural snowflake. Furthermore, when the latent graph can be represented in the feature space of a sufficiently regular kernel, we show that the combined neural snowflake and MLP encoder do not succumb to the curse of dimensionality by using only a low-degree polynomial number of parameters in the number of nodes. This implementation enables a low-dimensional isometric embedding of the latent graph. We conduct synthetic experiments to demonstrate the superior metric learning capabilities of neural snowflakes when compared to more familiar spaces like Euclidean space. 

Additionally, we carry out latent graph inference experiments on graph benchmarks. Consistently, the neural snowflake model achieves predictive performance that either matches or surpasses that of the state-of-the-art latent graph inference models. Importantly, this performance improvement is achieved without requiring random search for optimal latent geometry. Instead, the neural snowflake model achieves this enhancement in a differentiable manner.
\end{abstract}

\section{Introduction}
\label{s:Introduction}

Geometric deep learning~\citep{gdl_2017,bronstein2021geometric} is a rapidly developing field that expands the capabilities of deep learning to encompass structured and geometric data, such as graphs, pointclouds, meshes, and manifolds. Graph neural networks (GNNs) derive their knowledge primarily from the specific graph they operate on, but many real-world problems lack an accessible ground truth graph for computation. Latent graph inference aims to address this by dynamically inferring graphs through geometric representations. Existing models lack a strong theoretical foundation and use arbitrary similarity measures for graph inference, lacking principled guidelines. A key challenge is the absence of a differentiable method to deduce geometric similarity for latent graph inference. Recently the concept of neural latent geometry search has been introduced~\citep{borde2023neural}, which can be formulated as follows: given a search space $\mathfrak{R}$ denoting the set of all possible latent geometries, and the objective function $L_{T,A}(g)$ which evaluates the performance of a given geometry $g$ on a downstream task $T$ for a machine learning model architecture $A$, the objective is to optimize the latent geometry: $\inf_{g\in\mathfrak{R}} L_{T,A}(g).$ In the context of latent graph inference $\mathfrak{R}$ would denote the space of possible geometric similarity measures used to construct the latent graphs. Previous studies have utilized random search to find the optimal geometry in $\mathfrak{R}$~\citep{Kazi_2022,Manifold-dDGM}. However, these methods have their limitations as they cannot infer geometry in a differentiable manner, and the representation capabilities of Riemannian manifolds are constrained by certain assumptions inherent in their geometry.

\textbf{Contributions.} We introduce a trainable deep learning architecture which we can adaptively implement metrics on $\mathbb{R}^d$ spaces with a fractal-like geometry, called \textit{neural snowflakes}.  We prove that together a neural snowflake and a simple MLP encoder are enough to discover any latent graph geometry. In particular, the neural snowflake implements a fractal geometry on $\mathbb{R}^d$ in which any given finite latent weighted graph can be isometrically embedded and the elementary MLP implements that embedding. We show that in cases where the latent weighted graph has favourable geometry, the neural snowflake and MLP encoder break the curse of dimensionality by only requiring a polynomial number of parameters in the graph nodes to implement the isometric embedding.  We note the contrast with universal approximation theorems, e.g.\ \cite{yarotsky2017error,LuShenYangZhang_2021_UATRegularTargets,shen2022optimal,kratsios2022universal}, where the number of parameters required to implement a generic approximation depend exponentially on the dimension of the ambient space.   Our embedding results exhibit no such exponential dependence on the dimension of the ambient space.  We verify our theoretical guarantees experimentally with synthetic metric learning experiments and graph embedding tasks. Additionally we show that the neural snowflake and MLP encoder combination beat or match the state of the art across several latent graph inference benchmarks from the literature. This is achieved by learning the latent geometry in a differentiable manner, utilizing a single model.  Thus, the neural snowflake eliminates the need to conduct costly combinatorial searches across numerous combinations of potential embedding spaces.

\section{Background}
\label{s:Background}

\textbf{Related Work.} In the field of Geometric Deep Learning, most research has relied on human annotators or simple preprocessing algorithms to generate the graph structure used in GNNs. However, even when the correct graph is provided, it may not be optimal for the specific task, and the GNN could benefit from a rewiring process~\citep{Topping2021UnderstandingOA}. Latent graph inference allows models to dynamically learn the intrinsic graph structure of problems where the true graph is unknown~\citep{Wang2019DynamicGC,Kazi_2022}. This is particularly relevant in real-world applications where data might only be available in the form of a pointcloud. There are several works in the literature addressing latent graph inference. In particular, we can think of graph rewiring~\citep{ArnaizRodrguez2022DiffWireIG,Bi2022MakeHG,Guo2023HomophilyorientedHG,Topping2021UnderstandingOA} as a subset of latent graph inference in which an input graph is provided to the network, whereas latent graph inference in its most general form allows GNNs to infer a graph starting from only a pointcloud. When the underlying connectivity structure is unknown, traditional architectures like transformers~\citep{Vaswani2017AttentionIA} and attentional multi-agent predictive models~\citep{Hoshen2017VAINAM} use a fully-connected graph.  This assumption, however, leads to challenges when training with large graphs. Generating sparse graphs can offer computationally tractable solutions~\citep{Fetaya2018NeuralRI} and prevent over-smoothing~\citep{Chen2020MeasuringAR}. Various models have been proposed to tackle this problem, starting from Dynamic Graph Convolutional Neural Networks (DGCNNs)~\citep{Wang2019DynamicGC}, to approaches that separate graph inference and information diffusion, such as the Differentiable Graph Modules (DGMs) in~\cite{Cosmo2020LatentGraphLF} and~\cite{Kazi_2022}. Recent approaches have focused on generalizing the DGM leveraging product manifolds~\citep{Manifold-dDGM,borde2023projections}. Latent graph inference is also referred to as graph structure learning in the literature. A survey of similar methods can be found in~\cite{Zhu2021DeepGS}, and some classical methods include LDS-GNN~\citep{Franceschi2019LearningDS}, IDGL~\citep{Chen2020IterativeDG}, and Pro-GNN~\citep{Jin2020GraphSL}. Moreover, recently generalizing latent graph inference to latent topology inference~\citep{battiloro2023latent} has also been proposed.

\textbf{Graphs.} A weighted graph can be defined as an ordered pair $\mathcal{G} = (V,E,W)$, where $V$ represents a set of nodes (or vertices), $E\subseteq \{\{u,v\}:\,u,v\in V\}$ forms the collection edges (or links) within the graph, and $W:E\rightarrow (0,\infty)$ weights the importance of each edge.  An (unweighted) graph $\mathcal{G}$ is a weighted graph for which $W(\{u,v\})=1$ for every edge $\{u,v\}\in E$. The neighborhood $\mathcal{N}(v)$ of a node $v\in V$ is the set of nodes sharing an edge with $u$; i.e.\ $\mathcal{N}(v) \eqdef \{u\in V:\, \{u,v\}\in E\}$.

\textbf{Graph Neural Networks.} To compute a message passing \textit{Graph Neural Network} (GNN) layer over a graph $\mathcal{G}$ (excluding edge and graph level features for simplicity), the following equation is typically implemented: $\mathbf{x}_{i}^{(l+1)}=\phi\Big(\mathbf{x}_{i}^{(l)},\bigoplus_{j\in\mathcal{N}(x_{i}^{(l)})}\psi(\mathbf{x}_{i}^{(l)},\mathbf{x}_{j}^{(l)})\Big).$ In the given equation, $\psi\in\mathbb{R}^{d}\times\mathbb{R}^{d}\rightarrow\mathbb{R}^{h}$ represents a message passing function. The symbol $\bigoplus$ denotes an aggregation function, which must be permutation-invariant, e.g.\ the sum or max operation. Additionally, $\phi\in\mathbb{R}^{d}\times\mathbb{R}^{h}\rightarrow\mathbb{R}^{m}$ represents a readout function. We note that the update equation is local and relies solely on the neighborhood of the node. Both $\psi$ and $\phi$ can be Multi-Layer Perceptrons~(MLPs).  In our manuscript all MLPs will use the $\operatorname{ReLU}(t)\eqdef \max\{0,t\}$ activation function, where $t\in \mathbb{R}$.  
Several special cases have resulted in the development of a wide range of GNN layers: the most well-known being Graph Convolutional Networks (GCNs)~\citep{Kipf2017SemiSupervisedCW} and Graph Attention Networks (GATs)~\citep{Velickovic2018GraphAN}.

\textbf{Quasi-Metric Spaces.} While Riemannian manifolds have been employed for formalizing non-Euclidean distances between points, their additional structural properties, such as smoothness and infinitesimal angles, impose substantial limitations, rendering the demonstration of Riemannian manifolds with closed-form distance functions challenging. Quasimetric spaces, isolate the relevant properties of Riemannian distance functions without requiring any of their additional structure for graph embedding.  A \textit{quasi-metric space} is a set $X$ with a distance function $d:X\times X\rightarrow [0,\infty)$ satisfying for every $x,y,z\in X$: i) $d(x,y)=0$ if and only if $x=y$, ii) $d(x,y)=d(y,x)$, iii) $d(x,y)\le C\big(d(x,z)+d(z,y)\big)$, for some constant $C\ge 1$. When $C=1$, $(X,d)$ is called a metric space, examples include Banach spaces and also, every (geodesic) distance on a Riemannian manifold satisfies (i)-(iii).
Conversely, several statistical divergences are weaker structures than quasi-metrics since they fail (ii), and typically fail (iii); see e.g.~\citep[Proposition A.2]{Hawkins2017ATO}.  Property (iii) is called the \textit{$C$-relaxed triangle inequality} if $C>1$; otherwise (iii), is called the \textit{triangle inequality}. Quasi-metric spaces typically share many of the familiar properties of metric spaces, such as similar notions of convergence, uniform-continuity of maps between quasimetric spaces, and compactness results for functions between quasimetric spaces such as Arzela-Ascoli theorems
~\citep{xia2009geodesic}.  The next example of quasimetric spaces are called metric \textit{snowflakes}.
\begin{example}[{\cite{xia2009geodesic}}]
\label{ex:Snowflakes}
Let $p>0$ and $(X,d)$ be a metric space.  Then, $(X,d^p)$ is a quasimetric space with $C=2^{p-1}$ if $p>1$.  When $0<p\le1$, then $(X,d^p)$ is a metric space; whence, $C=1$.
\end{example}
\vspace{-.5em}
Snowflakes are a simple tool for constructing new (quasi) metric spaces from old ones with the following properties.  
Unlike products of Riemannian manifolds, a snowflake's geometry can be completely different than the original untransformed space's geometry.  
Unlike classical methods for constructing new distances from old ones, e.g.\ as warped products in differential geometry~\citep{chen1999warped,alexander1998warped}, snowflakes admit simple \textit{closed-form} distance functions.   

\begin{proposition}[{Snowflakes are Metric Spaces - \citep[Proposition 2.50]{WeaverLipschitzAlgebras_2ed_2018}}]
\label{prop:new_metrics}
Let $f:[0,\infty)\rightarrow [0,\infty)$ be a continuous, concave, monotonically increasing function with $f(0)=0$, and let $(X,d)$ be a metric space; then, $d_f:X\times X\rightarrow \mathbb{R}$ is a metric on $X$
$
        d_f(x,z)
    \eqdef 
        f(d(x,z))
,
$ for any $x,z\in X$.
\end{proposition}

\section{Adaptive Geometries via Neural Snowflakes}
\label{s:Neural_Snowflake}
We overcome one of the main challenges in contemporary Geometric Deep Learning, namely the problem of discovering a latent graph which maximizes the performance of a downstream GNN by searching a catalogue of combinations of products of elementary geometries~\citep{Gu2019LearningMR}; in an attempt to identify which product geometry the latent graph can be best embedded in~\citep{borde2023projections,Manifold-dDGM}. The major computational hurdle with these methods is that they pose a \textit{non-differentiable} combinatorial optimization problem with a non-convex objective, making them computationally challenging to scale. Therefore, by designing a class of metrics which are differentiable in their parameters, we can instead discover which geometry best suits a learning task using backpropagation.  Core to this is a trainable metric on $\mathbb{R}^d$, defined for any  $x,y\in \mathbb{R}^d$ by
\begin{equation}
\label{eq:activation_new}
        \|x-y\|_{\sigma_{\alpha,\beta,\gamma,p,C}}
    \eqdef 
        \big(
            \underbrace{
               C_1\,(
                   1
                    -
                    e^{-\gamma \|x-y\|}
                )
            }_{\text{Bounded}}
            +
            \underbrace{
                C_2\,
                \|x-y\|^{\alpha}
            }_{\text{Fractal}}
                +
            \underbrace{
                C_3\,
                \log(1+\|x-y\|)^{\beta}
            }_{\text{Irregular Fractal}}
        \big)^{1+|p|}
\end{equation}
where $0<\alpha,\beta\le 1$, $0\le p$, $0\le C_1,C_2,C_3,\gamma$ not all of which are $0$ and $C=(C_i)_{i=1}^3$. The trainable metric in~\eqref{eq:activation_new}, coined the \textit{snowflake activation}, is the combination of three components, a \textit{bounded geometry}, a \textit{fractal geometry}, and an \textit{irregular fractal} geometry part; as labeled therein.  The first \textit{bounded geometry} can adapt to latent geometries which are bounded akin to spheres, the second \textit{fractal geometry} component can implement any classical snowflake as in Example~\ref{ex:Snowflakes} (where $0<p\le 1$) and the \textit{irregular fractal} adapts to latent geometries much more irregular where the distance between nearby points grows logarithmically at large scales and exponentially at small scales\footnote{Note that $\log(1+\|x-y\|)\approx 1-e^{-\|x-y\|}$ when $0\approx \|x-y\|$.}. By Proposition~\ref{prop:new_metrics}, if $p=0$, the distance in~\eqref{eq:activation_new} is a metric on $\mathbb{R}^d$.  For $p>0$ $(\mathbb{R}^d,\|\cdot\|_{\sigma_{\alpha,\beta,\gamma,p,C}})$ is a quasi-metric space with $2^{p-1}$-relaxed triangle inequality, by Example~\ref{ex:Snowflakes}.

\subsection{Neural Snowflakes}
\label{ss:NeuralSnowflakes}
We leverage the expressiveness of deep learning, by extending the trainable distance function~\eqref{eq:activation_new} to a deep neural network generating distances on $\mathbb{R}^d$, called the \textit{neural snowflake}.

We begin by rewriting~\eqref{eq:activation_new} as a trainable activation function $\sigma_{a,b}:\mathbb{R}\rightarrow[0,\infty)$ which sends on any vector $u\in \mathbb{R}^J$, for $J\in \mathbb{N}_+$, to the $J\times 3$ matrix $\sigma_{a,b}(u)$ whose $j^{th}$ row is 
\begin{equation}
\label{eq:activation}
        \sigma_{a,b}(u)_{j}
    \eqdef 
        \big(
                1
                -
                e^{-|u_j|}
            ,
        |u_j|^{a}
            ,
                \log(1+|u_j|)^{b}
        \big)
.
\end{equation}
The parameters $0<a$ and $0\le b\le 1$ are trainable.  

We introduce a neural network architecture leveraging the ``tensorized'' snowflake activation function in~\eqref{eq:activation}, which can adaptively perturb any metric.  To ensure that the neural network model always preserves the metric structure of its input metric, typically the Euclidean metric on $\mathbb{R}^d$, we must constrain the weighs of the hidden layers to ensure that the model satisfies the conditions of Proposition~\ref{prop:new_metrics}.  
Building on the insights of monotone~\citep{daniels2010monotone}, ``input convex''~\citep{amos2017input} neural network architectures, and monotone-value~\citep{weissteiner2021monotone} neural networks, we simply require that all hidden weights are non-negative and do not all vanish. Lastly, the final layer of our neural snowflake model raises the generated metric to the $(1+|p|)^{th}$ power as in~\eqref{eq:activation_new}.  This allows the neural snowflake to leverage the flexibility of quasi-metrics, whenever suitable.  They key point here is that by only doing so on the final layer, we can explicitly track the relaxation of the triangle inequality discovered while training.  That is, as in Example~\ref{ex:Snowflakes}, $C=2^{p-1}$ if $p> 1$ and $C=1$ otherwise. Putting it all together, a \textit{neural snowflake} is a map $f:[0,\infty)\rightarrow[0,\infty)$, with iterative representation

\begin{equation}
\label{eq:neural_snowflake}
\begin{aligned}
    f(t) & =  t_{I}^{1+|p|}
    \\
        t_{i} 
    & =  
                B^{(i)}
                \,
                \sigma_{a_i,b_i}(A^{(i)}t_{i-1})
                \,
                C^{(i)}
        &\qquad \mbox{ for } i=1,\dots,I
    \\
    t_{0} & = t
\end{aligned}
\end{equation}
where for $i=1,\dots,I$, $A^{(i)}$ is a $\tilde{d}_{i}\times d_{i-1}$ matrix, $B^{(i)}$ is a $d_{i}\times \tilde{d}_{i}$-matrix, and $C^{(i)}$ is a $3\times 1$ matrix all of which have non-negative weights and at-least one non-zero weight, $p\in \mathbb{R}$; furthermore, for $i=1,\dots,I$, $0<a_i\le 1$, $0\le b_i\le 1$, $d_1,\dots,d_I\in \mathbb{N}_+$, and $d_0=1=d_{I}$.  We will always treat the neural snowflake as synonymous with the trainable distance function $\|x-y\|_f
    \eqdef 
        f(\|x-y\|),$ where $x,y\in \mathbb{R}^d$ for some contextually fixed $d$ and $f$ is as in~\eqref{eq:neural_snowflake}.

\section{Inferability Guarantees}
\label{s:Main_Result}
This section contains the theoretical guarantees underpinning the neural snowflake graph inference model.  We first show that it is universal, in the sense of graph representation learning, which we formalize.  We then derive a series of qualitative guarantees showing that the neural snowflake graph inference model requires very few parameters to infer any latent weighted graph.  In particular, neural snowflakes require a computationally feasible number of parameters to be guaranteed to work.

\subsection{Universal Graph Embedding}
\label{s:Main_Result__ss:Main_UniversalGraphInference}

Many graph inference pipelines depend on preserving geometry representations or encodings within latent geometries when inferring the existence of an edge between any two points (nodes) in a point cloud. Therefore, the effectiveness of any algorithm in this family of encoders hinges on its capacity to accurately or approximately represent the geometry of the latent graph. In this work, we demonstrate that the neural snowflake can infer any latent graph in this way.  Thus, we formalize what it means for a \textit{graph inference model} to be able to represent any latent (weighted) graph structure in $\mathbb{R}^D$ based on a class of geometries $\mathfrak{R}$.  
For any $D\in \mathbb{N}_+$, we call a pair $(\mathfrak{E},\mathfrak{R})$ a \textit{graph inference model}, on $\mathbb{R}^D$, if $\mathfrak{R}$ is a family of quasi-metric spaces, and $\mathfrak{E}$ is a family of maps with domain $\mathbb{R}^D$ and codomain in some member $(\mathcal{R},d_{\mathcal{R}})$ of $\mathfrak{R}$.  
Whenever $\mathbb{R}^D$ is clear from the context, we do not explicitly mention it.

\subsubsection{Universal Riemannian Representation Is Impossible}
\label{s:Main_Result__ss:EQuantitative_Guarantees___sss:Impossiblity}

Our primary qualitative guarantee asserts the universality of the graph inference model $(\mathfrak{E}, \mathfrak{R})$, where $\mathfrak{E}$ represents the set of MLPs into $\mathbb{R}^d$ with $\operatorname{ReLU}$ activation functions, and $\mathfrak{R}$ comprises all $(\mathbb{R}^d, |\cdot|_f)$, where $f$ is a neural snowflake; for integers $d\in \mathbb{N}_+$.  We now formalized what is meant by a \textit{universal} graph embedding model.
\begin{definition}[Universal Graph Embedding]
\label{defn:Universal_Graph_Embedding}
A graph inference model $(\mathfrak{E},\mathfrak{R})$ is \textit{universal} if: for every non-empty finite subset $V\subseteq \mathbb{R}^D$ and every connected weighted graph $\mathcal{G}=(V,E,W)$ there is a (quasi-metric) \textit{representation space} $(\mathcal{R},d_{\mathcal{R}})\in \mathfrak{R}$ and an \textit{encoder} $\mathcal{E}:\mathbb{R}^D\rightarrow \mathcal{R}$ in $\mathfrak{E}$ satisfying
\[
        d_{\mathcal{G}}(u,v)
    =
        d_{\mathcal{R}}(\mathcal{E}(u),\mathcal{E}(v))
    \qquad \forall\,u,v\in V
    .
\]
\end{definition}
Our interest in universal graph inference models lies in their ability to infer graph edges.  This is done by first learning an embedding $\mathcal{E}\in \mathfrak{E}$ into some representation space $(\mathcal{R},d_{\mathcal{R}})\in \mathfrak{R}$ and subsequently sampling edges based on nearest neighbors within the aforementioned embedding.

One technical point worth noting is that, when forming sets of nearest neighbors, ties between equidistant points are broken arbitrarily. This is accomplished by indexing (possibly randomly) the graph's vertices and selecting the first few nearest points based on the ordering of that index, similar to the approach in~\cite{fakcharoenphol2004approximating}. 

The formalization of this reconstruction procedure, in Theorem~\ref{thrm:ReconstructionTheorem}, uses the following notation. For every positive integer $N$, we denote the first $N$ positive integers by $[N]\eqdef \{1,\dots,N\}$.  For every quasi-metric (representation) space $(\mathcal{R},d_{\mathcal{R}})$ each point $x\in \mathcal{R}$, and each radius $r\ge 0$ the closed unit ball about $x$ of radius $r$ is $\bar{B}_{\mathcal{R}}(x,r)\eqdef \{u\in \mathcal{R}:\, d_{\mathcal{R}}(x,y)\le r\}$.

\begin{theorem}[Generic Graph Reconstruction via Universal Graph Inference Models]
\label{thrm:ReconstructionTheorem}
Fix $D \in \mathbb{N}_+$ and a latent graph inference model $(\mathfrak{E},\mathfrak{R})$ on $\mathbb{R}^D$.  
For every non-empty finite subset $V\subseteq \mathbb{R}^D$, every graph $G=(V,E)$, and each index $V=\{v_i\}_{i=1}^N$ there exists: a quasi-metric (representation) space $(\mathcal{R},d_{\mathcal{R}})\in \mathfrak{R}$ and an encoder $\mathcal{E}:\mathbb{R}^D\to \mathcal{R}$ in $\mathfrak{E}$ such that: 
for each $i\in [N]$ there is a (number of nearest neighbours) $k_i\in [N]$ satisfying
\[
        \{u_i,u_j\} \in E
    \Leftrightarrow
            j\le i^{\star}
        \mbox{ and }
                d_{\mathcal{R}}\big(
                        \mathcal{E}(u_i)
                    ,
                        \mathcal{E}(u_j)
                \big)
            \le 
                r(k_i)
\]
where $r(k_i)\eqdef \inf\big\{r\ge 0:\,
    \#\{ v\in V:\,
    d_{\mathcal{R}}\big(
            \mathcal{E}(u_i)
        ,
            \mathcal{E}(v)
    \big)
\le 
    r
\} \ge k_i
\big\}$ and where 
$i^{\star}\eqdef\{
j\in [N]:\,
\#(\bar{B}_{\mathcal{R}}(u_i, r(k_i) ) \cap \{u_s\}_{s=1}^j) \le k_i
\}$.
\end{theorem}
Theorem~\ref{thrm:ReconstructionTheorem} shows that if a latent graph inference model is universal, then it can be used to reconstruct the edge set of any latent graph structure by first embedding the vertices/point-cloud into a latent representation space and then joining nearest neighbours.  The next natural question is: \textit{``How does one build a universal latent graph inference model which is differentiable?''}

It is known that the family of Euclidean spaces $\mathfrak{R}=\{(\mathbb{R}^d,\|\cdot\|)\}_{d\in\mathbb{N}_+}$ are not flexible enough to isometrically accommodate all weighted graphs; even if $\mathfrak{E}$ is the family of \textit{all functions} from any $\mathbb{R}^D$ into any Euclidean space $(\mathbb{R}^d,\|\cdot\|)$.  This is because, some weighted graphs do not admit isometric embeddings into any Euclidean space~\citep{BourgainMilmanWolfson_TypeMetric_1986AMS,LinialLondonRabinovich_GeomGraphs_1995Combinatorics,MetousekEmbeddingintoLpExpanders_1997_IJM}.  Even infinite-dimensions need not be enough, since for every $n\in \mathbb{N}_+$, there is an $n$-point weighted graph which cannot be embedded in the Hilbert space $\ell^2$ with distortion less than $\Omega(\log(n)/\log(\log(n))$~\cite[Proposition 2]{Bourgain_IJM_1985}.  In particular, it cannot be isometrically embedded therein.  For example, any $l$-leaf tree embeds in $d$-dimensional Euclidean space with distortion at-least $\Omega(l^{1/d})$. In contrast, any finite tree can embed with arbitrary low-distortion into the hyperbolic plane~\citep{kratsios2023capacity}, which is a particular two-dimensional non-flat Riemannian geometry. Several authors~\citep{SarkarTreesEmbedding} have shown that cycle graphs can be embedded isometrically in spheres of appropriate dimension and radius~\citep{Schoenberg_Spheres_1942}, or in the product of spheres~\citep{GuellaMengattoPeron_ToriEmbedding_2016}, but they cannot be embedded isometrically in Euclidean space~\citep{Enflo_UniformHomeomorphisms_Nono_1976}.  These observations motivate geodesic deep learners \citep{Liu2019HyperbolicGN,Chamberlain2017NeuralEO,HGCN,Manifold-dDGM,borde2023projections} and network scientists \citep{Verbeek2014MetricEH} to use families of Riemannian representation spaces $\mathfrak{R}$ in which it is hoped that general graphs can faithfully be embedded, facilitating embedding-based latent graph inference. Unfortunately, $5$ nodes and $5$ edges are enough to construct a graph which cannot be isometrically embedded into \textit{any} non-pathological Riemannian manifold.

\begin{proposition}[Riemannian Representation Spaces are Too Rigid to be Universal]
\label{prop:Embedding_Impossible}
For any $D\in \mathbb{N}_+$ and any $5$-point subset $V$ of $\mathbb{R}^D$, there exists a set of edges $E$ on $V$ such that: 
\vspace{-0.5em}
\begin{itemize}
    \item[(i)] the graph $\mathcal{G}\eqdef (V,E)$ is connected 
    \item[(ii)] for every complete\footnote{Here, complete is meant in the sense of metric spaces; i.e.\ all Cauchy sequences in a complete metric space have a limit therein.} and connected smooth Riemannian manifold $(\mathcal{R},g)$ there does not exist an isometric embedding $\varphi:(V,d_{\mathcal{G}})\rightarrow (\mathcal{R},d_{\mathcal{R}})$
\end{itemize}
where $d_{\mathcal{G}}$ and $d_{\mathcal{R}}$ respectively denote the shortest path (geodesic) distances on $\mathcal{G}$ and on $(\mathcal{R},g)$.
\end{proposition}
In other words, Proposition~\ref{prop:Embedding_Impossible} shows that if $\mathfrak{R}$ is any set of non-pathological Riemannian manifolds and $\mathfrak{E}$ any set of functions from $\mathbb{R}^D$ into any Riemannian manifold $(\mathcal{R},g)$ in $\mathfrak{R}$ the graph inference model $(\mathfrak{E},\mathfrak{R})$ is not universal.  Furthermore, the graph ``breaking its universality'' is nothing obscure but a simple $5$ node graph.  Note that, no edge weights (not equal to $1$) are needed in Proposition~\ref{prop:Embedding_Impossible}.

\begin{figure}[hbpt!]
    \centering
    \includegraphics[width=.15\textwidth]{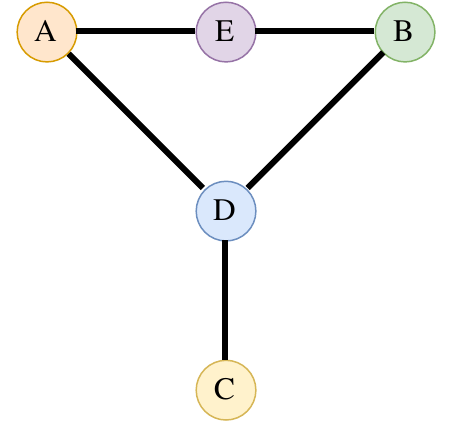}
    \caption{\textbf{Explanation of Proposition~\ref{prop:Embedding_Impossible}}: The Graph of Proposition~\ref{prop:Embedding_Impossible} cannot be isometrically embedded into any complete and connected (smooth) Riemannian manifold.  Briefly, the issue is that any isometric embedding into such a Riemannian manifold must exhibit a pair of geodesics one of which travels from the embeddings of node $C$ to node $A$, while passing through the embedding of node $D$; and likewise, the other of which travels from the embedding of node $C$ to node $B$ and again passes through the embedding of node $D$. However, this would violate the local uniqueness of geodesics in such a Riemannian manifold, around the embedding of node $D$ (implied by the Picard-Lindel\"{o}f theorem for ODEs); thus no such embedding can exist.}
    \label{fig:badgraph}
\end{figure}

Proposition~\ref{prop:Embedding_Impossible} improves on \citep[Propositions 13 and 15]{kratsiosembedding_2021} since the latter only show that no compact connected Riemannian manifolds (e.g.\ products of spheres) and no connected Riemannian manifold with bounded non-positive sectional curvatures (e.g.\ products of hyperbolic spaces) can accommodate certain sequences of expander graphs (see \citep[Remark 14]{kratsiosembedding_2021}).  However, those results do not rule out more complicated Riemannian representation spaces; e.g.\ the products of spheres, hyperbolic, and Euclidean spaces recently explored by~\cite{Gu2019LearningMR,Tabaghi2021LinearCI,heterogeneous_manifolds,Manifold-dDGM}. 

\subsubsection{Universal Representation Is Possible With Neural Snowflake}
\label{s:Main_Result__ss:EQuantitative_Guarantees___sss:Universality}

Juxtaposed against Proposition~\ref{prop:Embedding_Impossible}, our first main result shows that together, neural snowflake and MLPs, are universal graph embedding models.  
\begin{theorem}[{Universal Graph Embedding}]
\label{thrm:qualitative}
Let $D\in \mathbb{N}_+$ and $\mathcal{G}$ be a weighted graph with $V\subseteq \mathbb{R}^D$ with $I\in \mathbb{N}_+$ vertices.  There exists an embedding dimension $d\in \mathbb{N}_+$, an MLP $\mathcal{E}:\mathbb{R}^D\rightarrow \mathbb{R}^d$ with $\operatorname{ReLU}$ activation function, and a neural snowflake $f$ such that
\vspace{-.5em}
\[
\smash{
        d_{\mathcal{G}}(u,v) 
    = 
            \|
                    \mathcal{E}(u)
                -
                    \mathcal{E}(v)
            \|_f
,
}
\]
for each $u,v\in V$.
The $(\mathbb{R}^d,\|\cdot\|_f)$ supports a $2^{
O(\log(1+\frac{1}{I-1})^{-1} )}$-relaxed triangle inequality.  If $\mathcal{G}$ is a tree then $(\mathbb{R}^d,\|\cdot\|_f)$ instead supports an $8$-triangle inequality.
\end{theorem}

\textbf{Comparison: State-of-the-Art Deep Embedding Guarantees.} Recently, \cite{kratsiosembedding_2021} built on \cite{AndoniNaorNeiman_2018SnowflakeUniversalityWasserstein} and proposed a universal graph embedding model which uses the transformer architecture of~\cite{kratsios2021universal} to represent graphs in the order $2$-Wasserstein space on $\mathbb{R}^3$.  The drawbacks of this approach are that the metric is not available in closed-form, it is computationally infeasible to evaluate exactly for large graph embeddings, and it is still challenging to evaluate approximately~\citep{cuturi2013sinkhorn}. In contrast, Theorem~\ref{thrm:qualitative} guarantees that a simple MLP can isometrically embedding any weighted graph into a finite-dimensional representation space with closed-form distance function explicitly implemented by a neural snowflake.  

\textbf{Comparison: MLP without Neural Snowflake.} We examine the necessity of the neural snowflake in Theorem~\ref{thrm:qualitative}, by showing that the MLP alone cannot isometrically represent any weighted graph into its natural output space; namely some Euclidean space.  
\begin{theorem}[Neural Snowflakes $\&$ MLPs Are More Powerful For Representation Learning Than MLPs]
\label{thrm:strictbetterthanEuc}
Let $d,D\in \mathbb{N}_+$.  
The following hold:
\vspace{-0.5em}
\begin{enumerate}
    \item[(i)] \textbf{No Less Expressive Than MLP:} For any weighted graph $\mathcal{G}=(V,E,W)$ with $V\subseteq \mathbb{R}^D$, if there is an MLP $\mathcal{E}:\mathbb{R}^D\rightarrow\mathbb{R}^d$ which isometrically embeds $\mathcal{G}$ then there is a neural snowflake $f$ and an MLP $\mathcal{E}:\mathbb{R}^D\rightarrow\mathbb{R}^d$ which isometrically embeds $\mathcal{G}$ into $(\mathbb{R}^d,\|\cdot\|_f)$.
    \item[(ii)] \textbf{Strictly More Expressive Than MLP:} There exists a complete weighted graph $G=(V,E,W)$ with $V\subset\mathbb{R}^D$ by any MLP $\tilde{\mathcal{E}}:\mathbb{R}^D\rightarrow\mathbb{R}^d$ but for which there exists a neural snowflake $f$ and an MLP $\mathcal{E}:\mathbb{R}^D\rightarrow\mathbb{R}^d$ that isometrically embeds $\mathcal{G}$ into $(\mathbb{R}^D,\|\cdot\|_f)$.
\end{enumerate}
\end{theorem}

\subsection{Isometric Representation Guarantees - By Small Neural Snowflakes}
\label{s:Main_Result__ss:EQuantitative_Guarantees}
Theorem~\ref{thrm:qualitative} offers a qualitative assurance that the neural snowflake can represent any finite weighted graph.  We now show that any weighted graph which can be isometrically represented in the latent geometry induced by kernel regressors, can be implemented by a neural snowflake.  
We assume that the latent geometry of the weighted graph is encoded in a low-dimensional space and the distances in that low-dimensional space are given by a \textit{radially symmetric} and positive-definite kernel.

\begin{assumption}[Latent Radially-Symmetric Kernel]
\label{assumptions:boundedmetric_RBF}
There are $d,D\in \mathbb{N}_+$, a feature map $\Phi:\mathbb{R}^D\rightarrow \mathbb{R}^d$, and a non-constant positive-definite function\footnote{A positive-definite function, is map $f: \mathbb{R} \to \mathbb{C}$ for which each $\big(f(x_i - x_j)\big)_{i,j=1}^N$ is a positive-definite matrix, for every $N\in \mathbb{N}_+$ and each $x_1,\dots,x_N\in \mathbb{R}$.} $f:[0,\infty)\rightarrow[0,\infty)$ satisfying
\[
\smash{
        d_{\mathcal{G}}(x,y)
    =
        \bar{f}\big(
            \|\Phi(x)-\Phi(y)\|
        \big)
,
}
\]
for every $x,y\in\mathbb{R}^d$; where $\bar{f}(t)\eqdef f(0)-f(t)$.
\end{assumption}
Several satisfy Assumption~\ref{assumptions:boundedmetric_RBF}, as we emphasize using two exotic examples (Appendix~\ref{s:Further_Details}).
\begin{example}
\label{ex:Family_of_Examples}
The map $f(t)\eqdef (1+\sum_{k=1}^K\,|t|^{r_k})^{-\beta}$ satisfies Assumption~\ref{assumptions:boundedmetric_RBF} for any $K\in \mathbb{N}_+$ and $0\le r_1,\dots,r_K,\beta\le 1$.
\end{example}
\begin{example}
\label{ex:inverse_multiqudratics}
The map $f(t)\eqdef\exp\big(\frac{-a(t-1)}{\log(t)}\big)$ satisfies Assumption~\ref{assumptions:boundedmetric_RBF}, for all $a>0$.
\end{example}

We find that the neural snowflake can implement an isometric representation of the latent geometry, as in Theorem~\ref{thrm:qualitative}, using a small number of parameters comparable to Theorem~\ref{thrm:representation}.

\begin{theorem}[Quantitative Embedding Guarantees for Bounded Metric Geometries]
\label{thrm:bounded_component}
Let $D,d\in \mathbb{N}_+$, $G$ be a finite weighted graph with $V\subset \mathbb{R}^D$ and suppose that Assumption~\ref{assumptions:boundedmetric_RBF} (or~\eqref{assumptions:boundedmetric_MGF}) holds.  Then, there is a neural snowflake $(\mathcal{E},f)$ satisfying
\[
        \|\Phi(v)-\Phi(u)\|_f = \|\mathcal{E}(v)-\mathcal{E}(u)\|_{\bar{f}} 
\]
for every $u,v\in V$.  
Furthermore, the depth and width of $\mathcal{E}$ and $f$ are recorded in Table~\ref{tab:complexity_count}.
\end{theorem}

\subsection{Representation Guarantees Leveraging Distortion}
\label{s:Main_Result__ss:EQuantitative_Guarantees__wDist}

Theorem~\ref{thrm:qualitative} shows the neural snowflake and an simple MLP are \textit{universal graph embedding} models.
By allowing the neural snowflake and MLP some possible slack to distort the latent graph's geometry, either by stretching or contracting pairwise distances ever so slightly, we are able to derive explicit bounds on the embedding dimension $d$ and on the complexity of the neural snowflake and MLP.

\begin{theorem}[Quantitative Approximately Isometric Embeddings for Weighted Graphs]
\label{thrm:representation}
Let $G=(V,E,W)$ be a weighted graph $N\in \mathbb{N}_+$ nodes, a non-negative weighting function $W$, and $V\subset \mathbb{R}^d$.  For every $1<p<2$, there exists an MLP with ReLU activation function $\mathcal{E}:\mathbb{R}^d\rightarrow \mathbb{R}^{O(\log(\operatorname{dim}(G))}$ and a snowflaking network $f:\mathbb{R}\rightarrow\mathbb{R}$ such that: for every $x,u\in V$
\[
        d_G(x,u)
    \le 
        \|
            \mathcal{E}(x)
            -
            \mathcal{E}(u)
        \|_f
    \le 
        D^p
        \,
        d_G(x,u)
    ,
\]
where $D\in \mathcal{O}\big(\frac{\log(\operatorname{dim}(G))^2}{p^2}\big)$ and $(\mathbb{R}^n,\|\cdot\|_f$) is a quasi-metric space satisfying the $2^{p-1}$-triangle inequality.  The depth and width of $\mathcal{E}$ and $f$ are recorded in Table~\ref{tab:complexity_count}.
\end{theorem}

\begin{table}[H]
    \centering
    \caption{Complexity of the Neural Snowflake and MLPs. Here, $\mathcal{O}$ suppresses a constant depending only on $D$ and $C_{\mathcal{G}}>0$ depends only on $\mathcal{G}$ (explicit constants are given in Appendix~\ref{s:detailed_model_complexity}).}
    \label{tab:complexity_count}
    \begin{tabular}{l|lll|l}
    \toprule
      \textbf{Geometry} & \textbf{Net.} & \textbf{Hidden Layers (Depth $-1$)} & \textbf{Width} & 
      \textbf{Theorem} \\
      \midrule
      General & $f$ & $\Theta(1)$ & $\Theta(1)$ 
      & \ref{thrm:representation} \\
      Bounded & $f$ & $\Theta(1)$ & $
      \mathcal{O}\big(
        I^2
      \big)
      $ & 
      \ref{thrm:bounded_component} \\
      \midrule
      All & $\mathcal{E}$ & $
      \mathcal{O}\big(
        I
            \sqrt{I\log(I)}
                         \log\big(
                            I^2\,
                            C_{\mathcal{G}}
                          \big)
        \big)
      $ &  $
      \mathcal{O}\big(
        DI + d
      \big)
      $
      & \ref{thrm:bounded_component} and \ref{thrm:representation}
      \\
    \bottomrule
    \end{tabular}
\end{table}
\vspace{-1em}
\textbf{Discussion - Comparison with the Best Proven Deep Representation.} The complexity of the neural snowflake and MLP are reported in Table~\ref{tab:complexity_count}.
Theorem~\ref{thrm:representation} shows that the neural snowflake can match the best known embedding guarantees by a deep learning model with values in a curved infinite-dimensional space \cite[Theorem 4]{kratsiosembedding_2021}, both in terms of distortion and the number of parameters.  The neural snowflake's adaptive geometry allows for the representations to be implemented in $\mathcal{O}(\log(\operatorname{dim}(G))$ dimensions and the implemented representation space $(\mathbb{R}^d,\|\cdot\|_f)$ has an explicit closed-form distance function making it trivial to evaluate; unlike the distance function in the representation space of \cite{kratsiosembedding_2021}. Our guarantees show that neural snowflakes can theoretically represent any weighted graph; either isometrically or nearly isometrically with provable few parameters.

\section{Experimental Results}
\label{s:Experiments}

Next, we validate our embedding results through synthetic graph experiments, and we demonstrate how the combined representation capabilities of neural snowflakes and MLPs can enhance existing state-of-the-art latent graph inference pipelines.

\textbf{Synthetic Embedding Experiments: Neural Snowflakes vs Euclidean Graph Embedding Spaces.} To assess the effectiveness of neural snowflakes as well as to compare their performance with that of MLPs in approximating metrics using Euclidean space, we conduct synthetic graph embedding experiments. Specifically, we focus on fully connected graphs, where the node coordinates are randomly sampled from a multivariate Gaussian distribution within a 100-dimensional hypercube in Euclidean space, denoted as $\mathbb{R}^{100}$. The weights of the graphs are computed according to the metrics in Table~\ref{table:synthetic experiments}. In the leftmost column, the MLP model projects the node features in $\mathbb{R}^{100}$ to $\mathbb{R}^{2}$ and aims at approximating the edge weights of the graph using the euclidean distance $\|\textrm{MLP}(\mathbf{x}) - \textrm{MLP}(\mathbf{y})\|$ between the projected features. In the central column, the neural snowflake $f$ learns a quasi-metric based on the input Euclidean metric previously mentioned in $\mathbb{R}^{2}$, $f(\|\textrm{MLP}(\mathbf{x}) - \textrm{MLP}(\mathbf{y})\|)$. Finally, in the rightmost column the neural snowflake learns based on the Euclidean distance between features in $\mathbb{R}^{100}$: $f(\|\mathbf{x} - \mathbf{y}\|)$. In order to demonstrate the superior embedding capabilities of neural snowflakes compared to Euclidean spaces, we intentionally equip the MLP in the first column with a significantly larger number of model parameters than the other models. This is done to highlight the fact that, despite having fewer learnable parameters, neural snowflakes outperform Euclidean spaces by orders of magnitude. Furthermore, the results of our experiments reveal that neural snowflakes exhibit remarkable flexibility in learning metrics even in lower-dimensional spaces such as $\mathbb{R}^2$. This is evident from the similarity of results obtained in the rightmost column (representing $\mathbb{R}^{100}$) compared to those achieved in the case of learning the metric in $\mathbb{R}^{2}$. Additional information regarding these embedding experiments can be found in Appendix~\ref{s:Experimental Results Supplementary Material}.

\begin{table}[htbp]
\small
\caption{Results for synthetic graph embedding experiments for the test set. The Neural Snowflake models are able to learn the metric better with substantially lesser number of model parameters.} 
\centering
\scalebox{0.75}{ 
\begin{tabular}{lccc}
    \toprule
      &\textbf{MLP} & \textbf{Neural Snowflake (+ MLP)}  &   \textbf{Neural Snowflake}\\ 
    No. Parameters & 5422  & 4169  &  847 \\ 
    Embedding space, $\mathbb{R}^n$ & 2 &  2 &  100 \\\midrule
     Metric &  \multicolumn{3}{c}{Mean Square Embedding Error}  \\ \midrule
     $\|\mathbf{x} - \mathbf{y}\|^{0.5}\log(1+\|\mathbf{x} - \mathbf{y}\|)^{0.5}$ &  1.3196 & 0.0034 &  $\mathbf{0.0029}$ \\\midrule
     $\|\mathbf{x} - \mathbf{y}\|^{0.1}\log(1+\|\mathbf{x} - \mathbf{y}\|)^{0.9}$ &  1.2555 &  0.0032 &  $\mathbf{0.0032}$\\\midrule

     $1- \frac{1}{(1+\|\mathbf{x} - \mathbf{y}\|^{0.5})}$  & 0.1738  & $\mathbf{0.00004}$   & $\mathbf{0.00004}$
     \\\midrule

     $1 - \exp{\frac{-(\|\mathbf{x} - \mathbf{y}\|-1)}{\log(\|\mathbf{x} - \mathbf{y}\|)}}$  & 0.3249 & 0.00008 & $\mathbf{0.00008}$
     \\\midrule

     $1- \frac{1}{(1+\|\mathbf{x} - \mathbf{y}\|)^{0.2}}$  &  0.0315 &  $\mathbf{0.00005}$  & $\mathbf{0.00005}$\\\midrule

     $1- \frac{1}{1+\|\mathbf{x} - \mathbf{y}\|^{0.2}+\|\mathbf{x} - \mathbf{y}\|^{0.5}}$  & 0.2484 & $\mathbf{0.00002}$ & $\mathbf{0.00002}$\\

    \bottomrule
\end{tabular}
}
\label{table:synthetic experiments}
\end{table}

\textbf{Graph Benchmarks.} In this section, we present our findings on using Neural Snowflakes compared to other latent graph inference models. Our objective is to evaluate the representation power offered by various latent space metric spaces, and thus, we contrast our results with the original DGM~\citep{Kazi_2022} and its non-Euclidean variants~\citep{borde2023projections,Manifold-dDGM}. We also introduce a new variant of the DGM which uses a snowflake activation on top of the classical DGM, to equip the module with a snowflake quasi-metric space. We ensure a fair evaluation, by conducting all experiments using the same GCN model and only modify the latent geometries used for latent graph inference. We take care to maintain consistency in the number of model parameters, as well as other training specifications such as learning rates, training and testing splits, number of GNN layers, and so on. This approach guarantees that all comparisons are solely based on the metric (or quasi-metric) space utilized for representations. By eliminating the influence of other factors, we can obtain reliable and trustworthy experimental results. A detailed and systematic analysis of the results is provided in Appendix~\ref{s:Experimental Results Supplementary Material}. 

We first present results from latent graph inference on the well-known Cora and CiteSeer homophilic graph benchmarks. We use a consistent latent space dimensionality of 8 and perform the Gumbel top-k trick for edge sampling with a k value of 7. The models all share the same latent space dimensionality, differing solely in their geometric characteristics. In scenarios where a product manifold is used, the overall manifold is constructed by amalgamating two 4-dimensional manifolds through a Cartesian product. This yields a total latent space dimensionality of 8. This methodology ensures an equitable comparison based exclusively on geometric factors. All other parameters, comprising network settings and training hyperparameters, remain unaltered. For all DGM modules, GCNs are employed as the underlying GNN diffusion layers. Table~\ref{tab:homophilic_latent_space_dim_8_or} displays the results for Cora and CiteSeer, leveraging the original dataset graphs as inductive biases.

\begin{table*}[htbp!]
    \centering
    \caption{Results for Cora and CiteSeer leveraging the original input graph as an inductive bias.}
    \scalebox{0.75}{ 
    \begin{tabular}{lcccc}
    \toprule 
         &
         & 
         \textbf{Cora} &  
         \textbf{CiteSeer} & 
         
         \\ 
         \midrule
         
         Model  & Metric Space &
         \multicolumn{2}{c}{Accuracy $(\%)$ $\pm$ Standard Deviation}
         
         \\

        \midrule

        Neural Snowflake  & Snowflake  &
         $ \first{87.07 {\scriptstyle \pm 3.45}}$ &
         $ \first{74.76 {\scriptstyle \pm 1.74}}$ &\\

        DGM  & Snowflake  &
         $ 85.41 {\scriptstyle \pm 3.70}$ &
         $ \third{74.19 {\scriptstyle \pm 2.08}}$ &\\ \midrule

         DGM  & Euclidean  &
         $ \third{85.77 {\scriptstyle \pm 3.64}}$ &
         $ 73.67 {\scriptstyle \pm 2.30}$ &\\

         DGM  & Hyperboloid  &
         $ 85.25 {\scriptstyle \pm 3.34}$ &
         $ 73.46 {\scriptstyle \pm 1.79}$ &\\ 

         DGM  & Poincare  &
         $ \second{86.07 {\scriptstyle \pm 3.53}}$ &
         $ 71.23 {\scriptstyle \pm 5.53}$ &\\ 

         DGM  & Spherical  &
         $ 76.14 {\scriptstyle \pm 2.84}$ &
         $ 73.13 {\scriptstyle \pm 2.93}$ &\\ 

         DGM  & Euclidean $\times$ Hyperboloid  &
         $ 84.33 {\scriptstyle \pm 2.56}$ &
         $ 73.29 {\scriptstyle \pm 2.18}$ &\\

         DGM  & Hyperboloid $\times$ Hyperboloid  &
         $ 84.59 {\scriptstyle \pm 5.40}$ &
         $ \second{74.42 {\scriptstyle \pm 1.83}}$ &\\

         GCN & Euclidean  &
         $ 83.50 {\scriptstyle \pm 2.00}$ &
         $ 70.03 {\scriptstyle \pm 2.04}$ &\\

         \bottomrule
         
    \end{tabular}
    }
    
    \label{tab:homophilic_latent_space_dim_8_or}
\end{table*}

Next we perform experiments on Cora and CiteSeer without considering their respective graphs, that is, the latent graph inference models only take pointclouds as inputs in this case. We also include results for the Tadpole and Aerothermodynamics datasets used in~\cite{Manifold-dDGM}. Note that in these experiments all models used GCNs and a fixed latent space dimensionality of 8, unlike in the original paper which uses a larger latent spaces and GAT layers. The effect of changing the latent space dimensionality is further explored in Appendix~\ref{s:Experimental Results Supplementary Material}. For both Tadpole and Aerothemodynamics, the Gumbel Top-k algorithm samples 3 edges per node. See Table~\ref{tab:homophilic_latent_space_dim_8_nograph_or}.

\begin{table*}[htbp!]
    \centering
    \caption{Results for Cora and CiteSeer, and the real-world Tadpole and Aerothermodynamics datasets, without leveraging the original input graph as an inductive bias.}
    \scalebox{0.75}{ 
    \begin{tabular}{lcccccc}
    \toprule 
         &
         & 
         \textbf{Cora} &  
         \textbf{CiteSeer} & 
         \textbf{Tadpole} &
         \textbf{Aerothermodynamics}
         \\ 
         \midrule
         
         Model  & Metric Space &
         \multicolumn{4}{c}{Accuracy $(\%)$ $\pm$ Standard Deviation}
         
         \\

        \midrule

        Neural Snowflake  & Snowflake  &
         $ \first{71.22 {\scriptstyle \pm 4.27}}$ &
         $ \third{67.80 {\scriptstyle \pm 2.44}}$ &
         $ \third{90.02 {\scriptstyle \pm 3.51}}$&
         $ \third{88.65 {\scriptstyle \pm 3.10}}$ \\

        DGM  & Snowflake  &
         $ \third{69.51 {\scriptstyle \pm 4.42}}$ &
         $ 66.86 {\scriptstyle \pm 2.82}$ &
         $ \first{91.61 {\scriptstyle \pm 4.38}}$&
         $ 88.55 {\scriptstyle \pm 2.35}$  \\ \midrule

         DGM  & Euclidean  &
         $ 68.37 {\scriptstyle \pm 5.39}$ &
         $ \second{68.10 {\scriptstyle \pm 2.80}}$ &
         $ 89.29 {\scriptstyle \pm 4.66}$&
         $ 88.28 {\scriptstyle \pm 2.61}$  \\

         DGM  & Hyperboloid  &
         $ \second{70.00 {\scriptstyle \pm 4.08}}$ &
         $ \first{68.34 {\scriptstyle \pm 1.59}}$ &
         $ 88.75 {\scriptstyle \pm 5.11}$&
         $ 88.29 {\scriptstyle \pm 2.85}$  \\ 

         DGM  & Poincare  &
         $ 65.74 {\scriptstyle \pm 4.02}$ &
         $ 64.63 {\scriptstyle \pm 2.98}$ &
         $ 86.96 {\scriptstyle \pm 5.30}$&
         $ \second{89.00 {\scriptstyle \pm 2.34}}$  \\ 

         DGM  & Spherical  &
         $ 37.03 {\scriptstyle \pm 14.33}$ &
         $ 20.00 {\scriptstyle \pm 4.13}$ &
         $ 82.32 {\scriptstyle \pm 11.74}$&
         $ 88.37 {\scriptstyle \pm 3.00}$  \\ 

         DGM  & Euclidean $\times$ Hyperboloid  &
         $ 62.18 {\scriptstyle \pm 6.61}$ &
         $ 65.72 {\scriptstyle \pm 2.48}$ &
         $ \second{90.71 {\scriptstyle \pm 3.17}}$&
         $ 88.20 {\scriptstyle \pm 3.23}$  \\

         DGM  & Hyperboloid $\times$ Hyperboloid  &
         $ 67.14 {\scriptstyle \pm 4.19}$ &
         $ 64.33 {\scriptstyle \pm 9.44}$ &
         $ 89.64 {\scriptstyle \pm 5.57}$&
         $ \first{89.55 {\scriptstyle \pm 2.52}}$  \\
         
         MLP & Euclidean  &
         $ 58.92 {\scriptstyle \pm 3.28}$ &
         $ 59.48 {\scriptstyle \pm 2.14}$ &
         $ 87.70 {\scriptstyle \pm 3.46}$   &
         $ 80.99 {\scriptstyle \pm 8.34}$ 
         \\

         \bottomrule
         
    \end{tabular}
    }
    
    \label{tab:homophilic_latent_space_dim_8_nograph_or}
\end{table*}

From the results, we can see that models using snowflake metric spaces are consistently amongst the top performers for both Cora and CiteSeer. On the other hand, employing non-learnable metric spaces necessitates conducting an exploration of various latent space geometries to achieve optimal outcomes, since there is no metric space that consistently outperforms the rest regardless of the dataset. In our synthetic experiments, we have clearly demonstrated the remarkable advantage of neural snowflakes in learning metric spaces with enhanced flexibility, when compared to Euclidean space. 
In the context of latent graph inference, this distinction is not as pronounced as observed in the synthetic experiments. This can be attributed to the suboptimal nature of the Gumbel Top-k edge sampling algorithm (Appendix~\ref{s:Discrete Edge Sampling Algorithms}), a topic discussed in other research works~\citep{Kazi_2022,battiloro2023latent}, which essentially introduces a form of ``distortion'' to the learned metric. Yet, it is worth noting that the development of improved edge sampling algorithms to foster better synergy between metric space learning and graph construction is not the primary focus of this paper. Instead, our emphasis is on introducing a more comprehensive and trainable metric space and integrating it with existing edge sampling techniques.

\section{Conclusion}
\label{s:Conclusions}

Our theoretical analysis showed that a small neural snowflake, denoted as $f$, can adaptively implement fractal-like geometries $|\cdot|_f$ on $\mathbb{R}^d$, which are flexible enough to grant a small MLP the capacity to isometrically embed any finite graph. We showed that the non-smooth geometry implemented by the neural snowflake is key by showing that there are simple graphs that cannot be isometrically embedded into any smooth Riemannian representation space. We then explored several cases in which the combination of neural snowflakes and MLPs requires a small number of total parameters, independent of ambient dimensions, to represent certain classes of regular weighted graphs. We complemented our theoretical analysis by extensively exploring the best approaches to implement neural snowflakes, ensuring stability during training, in both synthetic and graph benchmark experiments. We also introduced a snowflake activation that can easily be integrated into the DGM module using a differentiable distance function, enabling the DGM to leverage a snowflake quasi-metric. We conducted tests on various graph benchmarks, systematically comparing the effectiveness of snowflake quasi-metric spaces for latent graph inference with other Riemannian metrics such as Euclidean, hyperbolic variants, spherical spaces, and product manifolds of model spaces. Our proposed model is consistently able to match or outperform the baselines while learning the latent geometry in a differentiable manner and without having to perform random search to find the optimal embedding space.

Note that our experiments were conducted in accordance with the differentiable graph module framework for discrete edge sampling as proposed by~\cite{Kazi_2022}. Our primary objective was to compare the representation capabilities of various (quasi-)metric spaces, while keeping all other model architecture choices constant. Recently, the NodeFormer~\citep{wu2023nodeformer} architecture was introduced, enabling the scalability of latent graph inference to large graphs. However, this development is slightly tangential to the research presented in this work, which focuses on analyzing the geometric characteristics of different embedding spaces. We propose considering the incorporation of the geometric notions discussed in this work into new scalable architectures, such as the NodeFormer, as part of future research.

\section*{Acknowledgment and Funding}
\label{s:AcknowledgmentandFunding}
AK acknowledges financial support from the NSERC Discovery Grant No.\ RGPIN-2023-04482 and their McMaster Startup Funds.  
The authors would also like to thank Giulia Livieri for her helpful feedback.

\bibliography{References}
\bibliographystyle{iclr2024_conference}

\newpage
\appendix

\section{Additional Background}
\label{s:Details_AddedBackground_MetricSpaces}
Many of the proofs or our paper's results, and generalizations thereof contained only in the manuscript's appendix, rest on some additional technology from the theory of metric spaces.  This brief appendix overviews those tools.  We also include formal definitions of the involved MLPs with $\operatorname{ReLU}$ activation function which we routinely use.

\subsection{Metric Spaces}
\label{s:Metric_Spaces}

Suppose that $(X,d)$ is a metric space; i.e.\ $C=1$ in (iii) above.  Topologically, $(X,d)$ and its snowflakes $(X,d^p)$ are identical but geometrically they are quite different.  Geometrically, the latter much more complex than the former.  In this paper we quantify complexity, when required, using the \textit{doubling dimension} and \textit{aspect ratio} of a metric space.

Denote a ball of radius $r\ge 0$ in $(X,d)$ about a point $x\in X$ is the set $B(x,r)\eqdef \{u\in X:\, d(x,u)<r\}$.  
The \textit{doubling dimension} of $(X,d)$, denoted by $\dim(X,d)$, is the smallest integer $k$ for which: for every ball $B(x,r)$ about any point $x\in X$ and radius $r>0$ there are $x_1,\dots,x_{2^k}\in X$ covering it by balls of half its radius; i.e.\ the metric ball $B(x,r)\subseteq\bigcup^{2^k}_{i=1}B(x_i,r/2)$. We emphasize that metric tools are finer then topological tools, for instance the dimension of many metric spaces\footnote{This is true for several sub-Riemannian manifolds, for instance} \cite{LeDonneRajala_2015_Dimensions} is often by much greater than their topological dimension.

We make use of the \textit{aspect ratio} of a weighted graph $\mathcal{G}=(V,E,W)$, as the ratio of the largest distance $d_{\mathcal{G}}$ between any two nodes over the smallest possible distance between any two nodes, with respect to the geodesic distance on $\mathcal{G}$.  This aspect ratio coincides with the aspect ratio used in \cite{kratsiosembedding_2021} and a metric variant of the aspect ratio of the measure theoretic aspect ratio of \cite{krauthgamer2005measured}.  Briefly, when $\#V>1$ we define
	\[
	\operatorname{aspect}(\mathcal{G})
	\eqdef 
	\frac{
		\max_{v,u\in V}\,
		d_{\mathcal{G}}(u,v)
	}{
		\min_{v,u\in V
			;\,
			u\neq v}\,
		d_{\mathcal{G}}(u,v)
	}
	,
\]
otherwise, we set $\operatorname{aspect}(\mathcal{G})=1$.  In the case where $V\subset \mathbb{R}^D$, for some $D\in \mathbb{N}_+$, and $W(u,v)=\|u-v\|$ we write $\operatorname{aspect}_2(V)\eqdef \operatorname{aspect}(\mathcal{G})$.

We often reply on the notion of a \textit{bi-Lipschitz} embedding, which we now recall.  
Given any $0<s\le L<\infty$, metric spaces $(X,d)$ and $(Y,\rho)$, and a map $f:\xxx\rightarrow \yyy$ is called \textit{$(s,L)$-bi-Lipschitz} if 
\begin{equation}
\label{eq:bi-Lipschitz}
        s\,
        d(u,v)
    \le 
        \rho(f(u),f(v))
    \le 
        L\,
        d(u,v)
\end{equation}
for each $u,v\in \xxx$.  If, $s=L=1$, then the map $f$ is said to be an isometric embedding\footnote{Some authors call the special case where $s=\frac1{L}$ an $L$-quasi-isometry.}.  

\subsection{MLPs with ReLU Activation Function}
\label{s:MLPs_ReLUActivation}
We will often be using Multi-layer Perceptron (MLPs) based on the perceptron model of \cite{rosenblatt1958perceptron}, and typically called feedforward neural networks in the contemporary approximation theory literature \cite{mhaskar2016deep,yarotsky2017error,petersen2018optimal,bolcskei2019optimal,elbrachter2021deep,galimberti2022designing,daubechies2022nonlinear,marcati2022exponential,adcock2020deep,shen2022optimal}.  Our MLPs will always use the $\operatorname{ReLU}(t)\eqdef \max\{0,t\}$ activation function, where $t\in \mathbb{R}$.  We will routinely apply the $\operatorname{ReLU}$ function \textit{component-wise}, any vector $u\in \mathbb{R}^d$ for any positive integer $N$, denoted by $\operatorname{ReLU}\sbullet[0.77] u$ and defined by
\begin{equation}\label{eq:component_wise}
        \sigma \sbullet[0.77] u 
    \eqdef  
        (\sigma(u_i))_{i=1}^{N}
    .
\end{equation}
For any pair of positive integers $D,d\in \mathbb{N}_+$, a map $f:\mathbb{R}^D\rightarrow\mathbb{R}^d$ is called an \textit{MLP} if it admits the recursive representation 
\begin{equation}\label{eq:feed_forward}
    \begin{split}
    f(x) & \eqdef  A^{(J)}\,x^{(J)} + b^{(J)},\\
     x^{(j+1)}& \eqdef \sigma \sbullet[0.75](A^{(j)}\,x^{(j)} + b^{(j)})\quad\text{for\,\,}j= 0, \ldots, J-1,\\
    x^{(0)} & \eqdef  {\color{black}{A^{(0)}}}x
.
    \end{split}
\end{equation}
for positive integers $d_0,\dots,\,d_J,\,d_{J+1},\,J$ with $d_0=d$ and $d_{J+1}=D$, $d_{j}\times d_{j+1}$-matrices $A^{(j)}$, and vectors $b^{(j)}\in \mathbb{R}^{d_{j+1}}$.  The depth of (the representation~\eqref{eq:feed_forward} of) $f$ is $D+1$, the number of \textit{hidden layers} of (the representation~\eqref{eq:feed_forward} of) $f$ is $D$, the width $W(f)$ of (the representation~\eqref{eq:feed_forward} of) $f$ is $\max_{j=0,\dots,J+1}\,d_j$, and the number of trainable/non-zero parameters $P(f)$ of (the representation~\eqref{eq:feed_forward} of) $f$ is 
\[
        P(f)
    \eqdef 
        \sum_{j=0}^{J+1}\,\|A^{(j)}\|_0 + \|b^{(j)}\|_0
    \le 
        W(f)^2\,(D+1)
\]
where $\|A^{(j)}\|_0$ (resp.\ $\|b^{(j)}\|_0$) counts the number of non-zero entries of $A^{(j)}$ (resp.\ of $b^{(j)}$).  
We remark that the estimate $P(f)\le W(f)^2\,(D+1)$ is often larger for several types neural networks architectures, e.g.\ convolutional neural networks with downsampling layers \cite{zhou2020theory,zhou2020universality}.

\section{Detailed Model Complexities}
\label{s:detailed_model_complexity}

This appendix contains a detailed version of Table~\ref{tab:complexity_count} with fully explicit constants.  

\begin{table}[H]
    \centering
    \caption{Details Complexity of the Neural Snowflakes and MLPs.}
    \label{tab:complexity_count__FullyExplicit}
    \resizebox{\columnwidth}{!}{%
    \begin{tabular}{l|lll|l}
      \textbf{Geometry} & \textbf{Net.} & \textbf{Hidden Layers (Depth $-1$)} & \textbf{Width} & 
      \textbf{Theorem} \\
      \midrule
      General & $f$ & $1$ & $3$ 
      & \ref{thrm:representation} \\
      Bounded & $f$ & $1$ & $\big\lceil \frac{I(I-1)+2}{4}\big\rceil$ & 
      \ref{thrm:bounded_component} \\
      \midrule
      All & $\mathcal{E}$ & $
      \mathcal{O}\left(
        I\left\{
            1+
                \sqrt{I\log(I)}
                    \,
                 \left[
                    1
                        +
                   \frac{\log(2)}{\log(I)}\,
                   \left(
                        C_D
                            +
                     \frac{
                         \log\big(
                            I^2\,
                            \operatorname{aspect}_2(V)
                          \big)
                     }{
                        \log(2)
                     }
                   \right)_+
                 \right]
            \right\}
        \right)
      $ &  $
      D(I - 1) + \max \{ d, 12\}
      $
      & \ref{thrm:bounded_component} and \ref{thrm:representation}
      \\
    \bottomrule
    \end{tabular}
    }
    \caption*{$I=\#V$ and the ``dimensional constant'' is $
        C_D
            \eqdef 
    \frac{
                2\log(5 \sqrt{2\pi})
            + 
                \frac{3}{2}
                \log(D)
            -
                \frac1{2}\log(D+1)
        }{
            2\log(2)
        }>0$.}
\end{table}

\section{Proofs}
\label{s:Proofs}

This section contains derivations of our paper's main results.

\subsection{Lemmata}
\label{s:Proofs__ss:Lemmata}
We can generate functions satisfying the conditions of Proposition~\ref{prop:new_metrics} using the following lemma.
\begin{lemma}[{\cite{448012}}]
\label{lem:characterization}
A function $f:[0,\infty)\rightarrow [0,\infty)$ is continuous, concave, and monotonically increasing if and only if there is a non-negative decreasing function $g:\mathbb{R}\rightarrow\mathbb{R}$ satisfying
\[
    f(t) = \int_0^t\,g(s)\,ds
\]
for each $t\in [0,\infty)$.
\end{lemma}
\begin{example}
For example, for any $0<a\le 1$ and $0\le b\le 1-a$, the map $f:[0,\infty)\rightarrow[0,\infty)$ given by $f(t)\eqdef t^a\ln(1+t)^b$ is continuous, concave, and monotonically increasing since 
\[
    f(t) = \int_0^t\, f(s)\, 
        \biggl(
            \frac{a}{s}
            +
            \frac{b}{(1+s)\log(1+s)}
        \biggr)
        \,
        ds
\]
and $t\mapsto f(t)\, 
        \big(
            \frac{a}{t}
            +
            \frac{b}{(1+t)\log(1+s)}
        \big)$ is a non-negative decreasing function on $\mathbb{R}$.
\end{example}

The following helpful lemma in constructing functions satisfying the conditions of Proposition~\ref{prop:new_metrics} from more elementary ones.
\begin{lemma}[{\citep[page 102]{BoydOptimizationBook}}]
\label{lem:Boyd}
Let $C\subseteq \mathbb{R}^d$ be a non-empty convex set and $d\in \mathbb{N}_+$.
If $f:C\rightarrow [0,\infty)$ is concave, and $g:[0,\infty)\rightarrow [0,\infty)$ is non-decreasing and concave, then $g\circ f:C\rightarrow\mathbb{R}$ is concave.  
If $C=[0,\infty)$ and if $f$ and $g$ are increasing, then so is $g\circ f$.
\end{lemma}

\begin{lemma}[Exponential Transformations]
\label{lem:elementary_snowflakingfunctions}
Let $a>0$.  The real-valued map $f$ on $[0,\infty)$ given for each $t\ge 0$ by $f(t)= 1-e^{-a\,t}$ satisfies the conditions of Proposition~\ref{prop:new_metrics}.
\end{lemma}
\begin{proof}[{Proof of Lemma~\ref{lem:elementary_snowflakingfunctions}}]
Since $\partial_t \, f(t) = a\,e^{-at}\ge a>0$ and $\partial_t^2\,f(t) = -a^2\,e^{-at}$, for every $t\in [0,\infty)$, then $f$ is strictly increasing and concave on $[0,\infty)$.  Noting that $f(0)=1-e^0=0$ completes the proof.
\end{proof}
Lemma~\ref{lem:elementary_snowflakingfunctions} directly implies that conical combinations of functions satisfying the conditions of Proposition~\ref{prop:new_metrics} also satisfy the condition of Proposition~\ref{prop:new_metrics}.
\begin{lemma}
\label{lem:conical_combinations}
If $N\in\mathbb{N}_+$, $A\in (0,\infty)^N$, and $f^{(1)},\dots,f^{(N)}$ satisfy the conditions of Proposition~\ref{prop:new_metrics}, then $f(t)\eqdef \sum_{n=1}^N\,A_n\,f^{(n)}(t)$ satisfies the conditions of Proposition~\ref{prop:new_metrics}.
\end{lemma}

\begin{lemma}[Classical Snowflaking Functions]
\label{lem:elementary_snowflakingfunctions_2}
Let $0<\alpha \le 1$.  The real-valued map $f$ on $[0,\infty)$ given for each $t\ge 0$ by $f(t)= t^{\alpha}$ satisfies the conditions of Proposition~\ref{prop:new_metrics}.
\end{lemma}
\begin{proof}[{Proof of Lemma~\ref{lem:elementary_snowflakingfunctions_2}}]
The case where $\alpha$ is clear; therefore, suppose that $\alpha<1$.
For every $t\in (0,\infty)$, we have that $\partial_t \, f(t) = \alpha t^{\alpha-1}>0$ and $\partial_t^2\,f(t) = (1-\alpha)\alpha t^{\alpha-2}=(1-\alpha^2)t^{\alpha-2}<(1-\alpha)t^{\alpha-2}<0$.  Therefore, $f$ is strictly increasing and concave on $[0,\infty)$.  Noting that $f(0)=0^{\alpha}=0$ completes the proof.
\end{proof}

\begin{lemma}[Logarithmic Snowflaking Functions]
\label{lem:elementary_logflakingfunctions}
Let $0<\beta \le 1$.  The real-valued map $f$ on $[0,\infty)$ given for each $t\ge 0$ by $f(t)= \log(1+|t|)^{\beta}$ satisfies the conditions of Proposition~\ref{prop:new_metrics}.
\end{lemma}
\begin{proof}[{Proof of Lemma~\ref{lem:elementary_logflakingfunctions}}]
Suppose that $\beta=1$.  Since $\log(1+|t|)=\log(1)=0$ then $f(0)=0$.
For every $t\in (0,\infty)$, we have that $\partial_t \, f(t) = \frac{1}{1+t}0$ and $\partial_t^2\,f(t) = -\frac{1}{1+t^2}<0$.  Therefore, $f$ is strictly increasing and concave on $[0,\infty)$.  By Lemma~\ref{lem:elementary_snowflakingfunctions} $t\mapsto t^{\beta}$ is increasing and concave on $(0,\infty)$ and since $f$ was strictly increasing and concave, then Lemma~\ref{lem:Boyd} yields the conclusion.
\end{proof}

\subsection{{Proof of Theorem~\ref{thrm:ReconstructionTheorem}}}
\label{s:ReconstructionTheorem__ss:Proof}

\begin{proof}[{Proof of Theorem~\ref{thrm:ReconstructionTheorem}}]
\hfill\\
\textit{Step 1 - Characterization of Edges by Cardinality of Geodesic Unit Balls}
\\
Let $n\eqdef \#V$.  
Fix any enumeration $V=\{v_i\}_{i=1}^n$.  
For each $i\in\{1,\dots,n\}$ set 
\[
k_i \eqdef \#\{w\in V:\, \{u_i,w\}\in E\}
.
\]
Note that, by definition, $k_i\in \{1,\dots,n\}$.
Since $G$ is unweighted, then $W(\{u_i,u_j\})=1$ for each $i,j\in [n]$.  
Consequentially, for each $i,j\in [n]$ if $\{u_i,u_j\}\in E$ then $(u_i=x_1,x_2=u_j)$ is a minimal path of length one, from $u_i$ to $u_j$.  Therefore, for each $i,j\in [n]$ we have that
\allowdisplaybreaks
\begin{align}
\label{eq:prop:combinatorial_inferable__shortest_path___begin}
        d_{G}(u_i,u_j) 
    &= 
        \inf_{(x_1,\dots,x_k)}\, 
            \sum_{i=1}^{k-1}\,
                W(\{x_i,x_{i+1}\})
    \\
    \nonumber
    &= 
        \inf_{(x_1,\dots,x_k)}\, 
            \sum_{i=1}^{k-1}\,
                1
    \\
    \nonumber
    &= 
        \inf_{(x_1,\dots,x_k)}\, 
            (k-1)
    \\
    \label{eq:prop:combinatorial_inferable__shortest_path}
    & = 1
\end{align}
where the infimum is taken over all paths $(u_i=x_1,\dots,x_k=u_j)$ on $G$ from $u_i$ to $u_j$. Thus,~\eqref{eq:prop:combinatorial_inferable__shortest_path___begin}-\eqref{eq:prop:combinatorial_inferable__shortest_path} imply that: for each $i\in [n]$
\begin{equation}
\label{eq:geometric_reformulation}
        \#\{w\in V:\, d_{\mathcal{G}}(u_i,w) =1 \} 
    = 
        \#\{w\in V:\, \{u_i,w\}\in E\}
    =
        k_i
    .
\end{equation}
By construction, for each $i\in [n]$, $r(k_i)=1$ and there exists exactly $k_i$ elements in $\overline{B}_{(V,d_G)}(u_i,r(k_i))\eqdef \{v\in V:\,d_G(u_i,v)\le r(k_i)\}$.  
Consequentially, 
\begin{equation}
\label{eq:geometric_reformulation_1a}
        \{u_i,u_j\} \in E
    \qquad
    \Leftrightarrow
    \qquad
            j\le i^{\star}
        \mbox{ and }
                d_{G}\big(
                        u_i
                    ,
                        u_j
                \big)
            \le 
                1
                =
                r(k_i)
.
\end{equation}
In particular, $i^{\star}=\#(\overline{B}_{(V,d_G)}(u_i,r(k_i))\cap V\}$; meaning that the condition $j\le i^{\star}$ can be dropped.  Thus,~\eqref{eq:geometric_reformulation_1a} simplifies to: for each $i\in [n]$ 
\begin{equation}
\label{eq:simplification}
        j\le i^{\star}
    \mbox{ and }
            d_{G}\big(
                    u_i
                ,
                    u_j
            \big)
        \le 
            r(k_i)
\qquad
\Leftrightarrow
\qquad
        d_{G}\big(
                u_i
            ,
                u_j
        \big)
    \le 
        r(k_i)
\end{equation}
Incorporating~\eqref{eq:simplification} into~\eqref{eq:geometric_reformulation_1a} allows us to further simplify to: for all $i,j\in [n]$
\begin{equation}
\label{eq:geometric_reformulation_2}
        \{u_i,u_j\} \in E
        \qquad
        \Leftrightarrow
        \qquad
            d_{G}\big(
                    u_i
                ,
                    u_j
            \big)
        \le 
            r(k_i)
\end{equation}
\textit{Step 2 - Reformulation via Embeddings}
\\
Since the graph inference model $(\mathfrak{E},\mathfrak{R})$ is universal, in the sense of Definition~\ref{defn:Universal_Graph_Embedding}, then there exists a quasi-metric space $(\mathcal{R},d_{\mathcal{R}})\in \mathfrak{R}$ and an encoder $\mathcal{E}:\mathbb{R}^D\to \mathcal{R}$ in $\mathfrak{E}$ such that for each $i,j\in [n]$
\begin{equation}
\label{eq:application_universality}
        d_{\mathcal{R}}\big(
            \mathcal{E}(u_i)
            ,
            \mathcal{E}(u_j)
        \big)
    =
        d_{G}\big(
            u_i
            ,
            u_j
        \big)
.
\end{equation}
Combining~\eqref{eq:application_universality} and~\eqref{eq:geometric_reformulation_2} implies that: for each $i,j\in [n]$ the pair $\{u_i,u_j\}$ belongs to $E$ if and only if
\begin{align}
\label{eq:geometric_reformulation__BEGINn}
            d_{\mathcal{R}}\big(
                    \mathcal{E}(u_i)
                ,
                    \mathcal{E}(u_j)
            \big)
        =
        d_{G}\big(
                u_i
                ,
                u_j
            \big)
        \le 
            1
            =
            r(k_i)
.
\end{align}
Combining~\eqref{eq:geometric_reformulation__BEGINn} with~\eqref{eq:simplification} yields: for each $i,j\in [n]$
\[
        \{u_i,u_j\}\in E
    \qquad
    \Leftrightarrow
    \qquad
            j\le i^{\star}
        \mbox{ and }
                    d_{\mathcal{R}}\big(
                        \mathcal{E}(u_i)
                    ,
                        \mathcal{E}(u_j)
                \big)
            \le 
                r(k_i)
.
\]
This concludes the proof.
\end{proof}

\begin{proof}[{Proof of Theorem~\ref{thrm:qualitative}}]
\textbf{Constructing The Isometric Embedding into a Euclidean Snowflake}
Since $(V,d_{\mathcal{G}})$ is an $I$-point metric space.  If $I=1$, there is nothing to show; suppose, therefore, that $I>1$.  Then, \citep[Corollary 3]{deza1990metric} implies that for $\varepsilon^{\star}\eqdef \log_2(1+\frac{1}{I-1})/2$ there exists some $d\in\mathbb{N}_+$ and a map $\tilde{\phi}_{\varepsilon^{\star}}:V\rightarrow \mathbb{R}^{d_{\varepsilon^{\star}}}$ satisfying
\begin{equation}
\label{PROOF_eq:thrm:qualitative__EmbeddingPerfect}
        d_{\mathcal{G}}(u,v)^{\varepsilon^{\star}}
    = 
        \|\phi(v)-\phi(u)\|
    .
\end{equation}
As shown in \cite{schoenberg1937certain},~\eqref{PROOF_eq:thrm:qualitative__EmbeddingPerfect} implies that the statement holds (mutatis mondanis) for any other $\varepsilon \in (0,\varepsilon^{\star}]$ for some other map $\tilde{\phi}_{\varepsilon}:V\rightarrow \mathbb{R}^{d_{\varepsilon}}$ into some Euclidean space $\mathbb{R}^{d_{\varepsilon}}$.  Fix any such $\varepsilon$, set $p\eqdef 1/\varepsilon$ and let $\varphi$ by any extension of $\tilde{\varphi}_{\varepsilon}$ defined on all of $\mathbb{R}^D$.   

\citep[Lemma 20]{kratsiosembedding_2021} implies that is an MLP with $\operatorname{ReLU}$ activation function $E:\mathbb{R}^D\rightarrow\mathbb{R}^d$ satisfying $\mathcal{E}(u)=\varphi(u)$ for all $u\in V$.  Whence,~\eqref{PROOF_eq:thrm:qualitative__EmbeddingPerfect} implies that
\[
        d_{\mathcal{G}}(u,v)
    = 
        \|\mathcal{E}(v)-\mathcal{E}(u)\|^{p}
.
\]
NB, in the special case where $\mathcal{G}$ is a tree \citep[Theorem 6]{maehara1986metric} we may instead set $\varepsilon^{\star}=1/2$ and therefore $p$ may be instead taken to be $p=4$ in the above argument.
\textbf{Implementing The Snowflaking Function}
We now implement the snowflaking map $t\mapsto t^p$ using a neural snowflake $f$; i.e.\ with representation~\eqref{eq:neural_snowflake}.  
Set $I=1$, let $p=2/\big( \log_2(1+\frac{1}{I-1}) \big)$, and consider the parameters 
\allowdisplaybreaks
\begin{align*}
    A^{(1)}=(1)
    ,\,
    B^{(1)}=(1)
    ,\,
    C^{(1)} = \begin{pmatrix}
        0\\
        1\\
        0
    \end{pmatrix}
    \alpha = 1
    ,\,
    \beta=0
.
\end{align*}
Then, $f$ defined by~\eqref{eq:neural_snowflake} satisfies $f(t)=t^p$.  Furthermore, $f$ has $1$ hidden layer, width $3$, and $4$ trainable parameters and~\eqref{PROOF_eq:thrm:qualitative__EmbeddingPerfect} implies that
\[
        d_{\mathcal{G}}(u,v)
    = 
        f\big(\|\mathcal{E}(v)-\mathcal{E}(u)\|\big)
.
\]
Since $p=2/\big( \log_2(1+\frac{1}{I-1}) \big)$ then Example~\ref{ex:Snowflakes} show that $f(\|\cdot-\cdot\|)$ is a quasi-metric space with $2^{2/\big( \log_2(1+\frac{1}{I-1}) \big)-1}$-relaxed triangle inequality.  

In the special case where $\mathcal{G}$ is a tree, we have $p=4$.  Thus, the neural snowflake supports an $8$-relaxed triangle inequality.
\end{proof}

\begin{proof}[{Proof of Theorem~\ref{thrm:representation}}]
\textbf{Existence and Memorization of the Snowflake Embedding}
Set $\varepsilon\eqdef 1-p^{-1}$ and note that $\varepsilon \in (1/2,1)$.  By \citep[Theorem 2]{NaorNeimannAssouadEmbedding_2012}, there exists $D,N>0$ and a map $\varphi:(V,d_G^{1-\varepsilon})\rightarrow \mathbb{R}^N$ satisfying: for every $x,u\in V$
\begin{equation}
\label{prooof_thrm:representation__eq:NaorNeimannAssouadEmbedding}
        d_G(x,u)^{1-\varepsilon}
    \le 
        f\big(
            \|
                \phi(x)
                -
                \phi(u)
            \|
        \big)
    \le 
        D
        \,
        d_G(x,u)^{1-\varepsilon}
    ,    
\end{equation}
where $N=c_1\,\log(\operatorname{dim}(G))$ and $D=c_2\,\frac{\log(\operatorname{dim}(G))^2}{\varepsilon^2}$, for absolute constants $c_1,c_2>0$ independent of $G$ and of $p$.  Set $f(t)\eqdef t^p=t^{1/(1-\varepsilon)}$.  Since $f$ is monotone increasing, then applying it through the inequalities in~\eqref{prooof_thrm:representation__eq:NaorNeimannAssouadEmbedding} yields
\begin{equation}
\label{prooof_thrm:representation__eq:NaorNeimannAssouadEmbedding_DONE}
        d_G(x,u)
    \le 
        \|
            \phi(x)
            -
            \phi(u)
        \|
    \le 
        D^{p}
        \,
        d_G(x,u)
    ,    
\end{equation}
for every $x,u\in V$.

By \citep[Lemma 20]{kratsiosembedding_2021}, there exists an MLP $\mathcal{E}:\mathbb{R}^d\rightarrow\mathbb{R}^D$ with $\operatorname{ReLU}$ activation function such that, for every $x\in V$ we have $\phi(x)=\mathcal{E}(x)$.  Therefore,~\eqref{prooof_thrm:representation__eq:NaorNeimannAssouadEmbedding_DONE} implies that
\begin{equation}
\label{prooof_thrm:representation__eq:NaorNeimannAssouadEmbedding_INTERPOL}
        d_G(x,u)
    \le 
        \|
            \phi(x)
            -
            \phi(u)
        \|
    \le 
        D^{p}
        \,
        d_G(x,u)
    ,    
\end{equation}
for each $u,u\in V$.  Furthermore, the depth, width, and number of trainable parameters defining $\mathcal{E}$ are as in \citep[Lemma 20]{kratsiosembedding_2021} and a recorded in Table~\ref{tab:complexity_count__FullyExplicit} (with abbreviated versions recorded in Table~\ref{tab:complexity_count}).

\textbf{Implementing The Snowflaking Function}
As in the proof of Theorem~\ref{thrm:qualitative}, we now implement the snowflaking map $t\mapsto t^p$ using a neural snowflake $f$; i.e.\ with representation~\eqref{eq:neural_snowflake}.  
Set $I=1$, let $p$ in~\eqref{eq:neural_snowflake} be the as in~\eqref{prooof_thrm:representation__eq:NaorNeimannAssouadEmbedding_DONE}, and consider the parameters 
\allowdisplaybreaks
\begin{align*}
    A^{(1)}=(1)
    ,\,
    B^{(1)}=(1)
    ,\,
    C^{(1)} = \begin{pmatrix}
        0\\
        1\\
        0
    \end{pmatrix}
    \alpha = 1
    ,\,
    \beta=0
.
\end{align*}
Then, $f$ defined by~\eqref{eq:neural_snowflake} satisfies $f(t)=t^p$.  Furthermore, $f$ has $1$ hidden layer, width $3$, and $4$ trainable parameters. 
\end{proof}

\subsection{Bounded Component}

For the proof of the next result, we recall that a function $f:[0,\infty)\rightarrow[0,\infty)$ is said to be \textit{completely monotone} if $f$ is continuous on $[0,\infty)$, smooth on $(0,\infty)$, and its derivatives satisfy the following alternating sum property
\begin{equation}
\label{eq:completemonotonicity}
        (-1)^n\partial_t^n\,f(t)
    \ge 
        0
\end{equation}
for every $t> 0$ and every $n\in \mathbb{N}_+$.  See \citep[Chapter IV]{Widder_LaplaceTransformBook1941} for a detailed study of completely monotone functions and several examples thereof.  

In particular, this and Proposition~\ref{prop:new_metrics} imply that $-f$ is monotonically increasing and concave, therefore the map, $\bar{f}$, defined for $t\ge 0$ by $\bar{f}(t)\eqdef f(0)-f(t)$ produces a well-defined snowflake metric $d_{\bar{f}}\eqdef \bar{f}(\|\cdot-\cdot\|)$.  Since $f$ is completely monotone then it is monotonically decreasign on $[0,\infty)$ and bounded below by $0$; whence, $\bar{f}(t)$ is monotonically increasing and contained in $[0,f(0)]$.  Therefore, $d_{\bar{f}}$ is \textit{bounded} between $[0,f(0)]$.  We will show that the neural snowflake can generate a snowflake metric which interpolates any bi-Lipschitz embedding into such a space.

Before proving Theorem~\ref{thrm:bounded_component} we note that it holds under the following alternative assumption to Assumption~\ref{assumptions:boundedmetric_RBF}.  Intuitively, this assumptions state that the map $f$, in Assumption~\ref{assumptions:boundedmetric_RBF} can be taken to be the moment-generating function (MGF) of some probability measure on $[0,\infty)$.  
\begin{assumption}[Alternative to Assumption~\ref{assumptions:boundedmetric_RBF}: Latent MGD Geometry]
\label{assumptions:boundedmetric_MGF}
There are $d,D\in \mathbb{N}_+$, a latent feature map $\Phi:\mathbb{R}^D\rightarrow \mathbb{R}^d$, and a Borel probability measure $\mathbb{P}$ on $[0,\infty)$ whose MGF $f(t) = \mathbb{E}_{X\sim \mathbb{P}}[ e^{-t\,X}]$ exists for all $t\ge 0$, is non-constant, and satisfies
\[
        d_{\mathcal{G}}(x,y)
    =
        \bar{f}\big(
            \|\Phi(x)-\Phi(y)\|
        \big)
,
\]
where $\bar{f}(t)\eqdef f(0)-f(t)$.
\end{assumption}

Both Assumptions~\ref{assumptions:boundedmetric_RBF} and~\ref{assumptions:boundedmetric_MGF} are special cases of the following more general assumption.
\begin{assumption}[Kernel or Moment-Generating Priors]
\label{assumptions:boundedmetric}
Suppose that $f:[0,\infty)\rightarrow[0,\infty)$ is non-constant and either:
\begin{enumerate}
    \item[(i)] $f(t) = \int_0^{\infty}\,e^{-t\,u}\,\mu(du)$ for some finite Borel measure $\mu$ on $[0,\infty)$,
    \item[(ii)] For each $k\in\mathbb{N}_+$, the map $K_f:(x,y)\in \mathbb{R}^k\times \mathbb{R}^k \mapsto f(x^{\top}y)\in \mathbb{R}$ is a positive-definite kernel on $\mathbb{R}^k$.
\end{enumerate}
Define $\bar{f}(t)\eqdef f(0)-f(t)$.
\end{assumption}

The following result implies, and generalizes, Theorem~\ref{thrm:bounded_component}, to bi-Lipschitz embeddings.  
\begin{proposition}[Embeddings into Bounded Metric Spaces - Bi-Lipschitz Version]
\label{prop:bounded_component__bi-LipschitzVersion}
Let $D,d\in \mathbb{N}_+$, $f$ satisfy Assumption~\ref{assumptions:boundedmetric}, and $G$ be a finite weighted graph with $V\subset \mathbb{R}^d$.  Then, $\|\cdot-\cdot\|_{\bar{f}}$ defines a bounded metric on $\mathbb{R}^d$ and for every $(s,L)$-bi-Lipschitz embedding $\Phi:(V,d_G)\rightarrow(\mathbb{R}^d,\|\cdot\|_{\bar{f}})$ there is a neural snowflake $(\mathcal{E},f)$ satisfying
\[
        \|\Phi(v)-\Phi(u)\|_f = \|\mathcal{E}(v)-\mathcal{E}(u)\|_{\bar{f}} 
\]
for every $u,v\in V$.  In particular, for each $u,v\in V$ we have
\[
        s\,d_G(v,u)
    \le 
        \|\mathcal{E}(v)-\mathcal{E}(u)\|_f 
    \le 
        s\,L\,
        d_G(v,u)
    .
\]
Furthermore, the depth, width, and number of trainable parameters of $\mathcal{E}$ and of $f$ are as in Table~\ref{tab:complexity_count}.
\end{proposition}

\begin{proof}[{Proof of Theorem~\ref{prop:bounded_component__bi-LipschitzVersion}}]
\textbf{Rephrasing as completely monotone functions}
Suppose that Assumption~\ref{assumptions:boundedmetric} holds.  The Hausdorff-Bernstein-Widder theorem, see \citep[Theorem IV.12a]{Widder_LaplaceTransformBook1941}, implies that $f$ is completely monotone.  Alternatively, suppose that Assumption~\ref{assumptions:boundedmetric} (ii) holds then Schoenberg's theorem, see\footnote{See \cite{GenShoenberg_2019_AA} for more general version.} \citep[Theorem 3]{Schoeberg_1938_CompletelyMonotoneMetricSpaces_AnnMath}, implies that $f$ is completely monotone.  

Since $f$ is defined on all of $[0,\infty)$ then~\eqref{eq:completemonotonicity} implies that $f(t)\ge 0$ for all $t$ whence takes non-negative valued.  Likewise~\eqref{eq:completemonotonicity} implies that $\partial_t\,f(t)\le 0$ therefore $f$ is non-increasing; whence $\bar{f}$ is bounded in $[0,f(0)]$.  Finally,~\eqref{eq:completemonotonicity} implies that $\partial_t^2\,f(t)\ge 0$ therefore $f$ is convex; thus $\bar{f}$ is concave since $\partial_t^2\,\bar{f}\le 0$.  Whence Proposition~\ref{prop:new_metrics} implies that $\|\cdot-\cdot\|_{\bar{f}}$ is a metric on $\mathbb{R}^D$.

\textbf{Interpolating $\mathcal{E}$ and $\bar{f}$}\\
Enumerate $V=\{x_i\}_{i=1}^I$ where $I\eqdef \#I$.  Consider an $(s,L)$-bi-Lipschitz embedding $\Phi:(V,d_G)\rightarrow(\mathbb{R}^d,\|\cdot\|_f)$.  By \citep[Lemma 20]{kratsiosembedding_2021}, there exists an MLP $\mathcal{E}:\mathbb{R}^D\rightarrow\mathbb{R}^d$ with $\operatorname{ReLU}$ activation function satisfying
\begin{equation}
\label{PROOF_thrm:bounded_component__eq:MLP_interpol_V}
\Phi(v)=\mathcal{E}(v)
\end{equation}
for each $v\in V$.  
Moreover, the depth, width, and number of trainable parameters determining $E$ are as in \citep[Lemma 20]{kratsiosembedding_2021} and are recorded in Table~\ref{tab:complexity_count__FullyExplicit} (and abbreviated in Table~\ref{tab:complexity_count}).

Since $\Phi$ is bi-Lipschitz then it is injective; whence $\{\|\Phi(x_i)-\Phi(x_j)\|\}_{i,j=1}^I$ has exactly as many points as $\{\|x_i-x_j\|\}_{i,j=1}^I$. By symmetry of the Euclidean metric, observe that the number $\tilde{I}$ of elements in set $\{\|\Phi(x_i)-\Phi(x_j)\|\}_{i,j=1}^I\cup \{0\}$ is at-most $I(I-1)/2 +1$ elements; which we sort and enumerate $\{\|x_i-x_j\|\}_{i,j=1}^I\cup \{0\}$ by $\{t_i\}_{i=1}^{\tilde{I}}$. 
By the \cite[Corollary on page 215]{mcglinn1978uniform} there exists a unique exponential sum $Y(t)=\sum_{i=1}^{\lceil \tilde{I}/2\rceil}\,\beta_i\,e^{-\alpha_i\,t}$ satisfying
\begin{equation}
\label{PROOF_thrm:bounded_component__eq:MLP_interpol}
    Y(t_i) = f(t_i)
\end{equation}
for every $i=0,\dots, \tilde{I}$, and in particular $Y(0)=f(0)$ since $0\in \{t_i\}_{i=1}^{\tilde{I}}$.  Since $\bar{f}(t)=f(0)-f(t)$ for all $t\ge 0$, then~\eqref{PROOF_thrm:bounded_component__eq:MLP_interpol} implies that
\allowdisplaybreaks
\begin{align}
\label{PROOF_thrm:bounded_component__eq:MLP_interpol_t__BEGIN}
    \bar{f}(t_i) = & f(0)-f(t_i)
\\
\nonumber
        = & Y(0)-Y(t_i)
\\
\nonumber
    = & \biggl(\sum_{i=1}^{\lceil \tilde{I}/2\rceil}\,\beta_i\,e^{-\alpha_i\,0} \biggr)- \biggl(\sum_{i=1}^{\lceil \tilde{I}/2\rceil}\,\beta_i\,e^{-\alpha_i\,t_i}\biggr)
\\
\nonumber
    = & \sum_{i=1}^{\lceil \tilde{I}/2\rceil}\,
        \Big(
                \beta_i\,e^{-\alpha_i\,0} 
            -
                \beta_i\,e^{-\alpha_i\,t_i}
        \Big)
\\
\nonumber
    = & \sum_{i=1}^{\lceil \tilde{I}/2\rceil}\,
        \Big(
                \beta_i\, 1
            -
                \beta_i\,e^{-\alpha_i\,t_i}
        \Big)
\\
\label{PROOF_thrm:bounded_component__eq:MLP_interpol_t__END}
    = & \sum_{i=1}^{\lceil \tilde{I}/2\rceil}\,
        \beta_i
            \Big(
                    1
                -
                    e^{-\alpha_i\,t_i}
            \Big)
    \eqdef \bar{Y}(t_i)
,
\end{align}
for all $i=1,\dots,\tilde{I}$.
In particular,~\eqref{PROOF_thrm:bounded_component__eq:MLP_interpol_t__BEGIN}-\eqref{PROOF_thrm:bounded_component__eq:MLP_interpol_t__END}, the definition of the $t_i$, and the memorization/interpolation guarantee in~\eqref{PROOF_thrm:bounded_component__eq:MLP_interpol_V} implies that
\allowdisplaybreaks
\begin{align}
\label{PROOF_thrm:bounded_component__eq:MLP_interpol_t}
        f\big(
            \|\Phi(x_i)-\Phi(x_j)\|
        \big)
    = 
        \bar{Y}\big(
            \|\mathcal{E}(x_i)-\mathcal{E}(x_j)\|
        \big)
    = 
        \bar{Y}\big(
            \|\mathcal{E}(x_i)-\mathcal{E}(x_j)\|
        \big)
\end{align}
for $i,j=1,\dots,I$ (note $0=\|\mathcal{E}(x_1)-\mathcal{E}(x_1)\|$ so $f(0)=Y(0)$ is implied by~\eqref{PROOF_thrm:bounded_component__eq:MLP_interpol_t}).  Therefore, for each $i,j=1,\dots,I$ we have
\allowdisplaybreaks
\begin{align*}
        \|\mathcal{E}(x_i)-\mathcal{E}(x_j)\|_{\bar{f}}
    = &
        \|\Phi(x_i)-\Phi(x_j)\|_{\bar{f}}
    \\
    \eqdef &
        \bar{f}\big(
            \|\Phi(x_i)-\Phi(x_j)\|
        \big)
    \\
    = & 
            \bar{Y}\big(
                \|\Phi(x_i)-\Phi(x_j)\|
            \big)
    \\
    = & 
        \|\Phi(x_i)-\Phi(x_j)\|_{\bar{Y}}
.
\end{align*}
It remains to show that $\bar{Y}$ can be implemented by a map, which we denote by $\tilde{f}$, with representation~\eqref{eq:neural_snowflake}.

\textbf{Implementing The Exponential Sum}
In the notation of~\eqref{eq:neural_snowflake}, set $I=1$ and consider
\begin{equation}
\label{PROOF_eq:parameterizationBoundedPortion}
\begin{aligned}
    A^{(1)} = \begin{pmatrix}
        \alpha_1 \\
        \vdots \\
        \alpha_M
    \end{pmatrix}
    ,
    \,
    B^{(1)}=\begin{pmatrix}
        \beta_1 \dots \beta_M
    \end{pmatrix}
    ,
    \,
    C^{(1)}=\begin{pmatrix}
        1 \\
        0\\
        0
    \end{pmatrix}
    ,\,
    d = \begin{pmatrix}
        0
    \end{pmatrix}
    ,\,
    \alpha_1=0=\beta_0
.
\end{aligned}
\end{equation}
Then, the map $\tilde{f}$ defined as in~\eqref{eq:neural_snowflake} with parameters~\eqref{PROOF_eq:parameterizationBoundedPortion} is
\[
    \tilde{f}(t)  =  
    \big(
        B^{(1)}\,\sigma_{0,0}(A^{(1)}t)C^{(1)}
    \big)^{1+|0|}
    =
    \sum_{i=1}^M\, \beta_i\,(1-e^{-\alpha_i\,t})
    =
    \bar{Y}(t)
\]
for every $t\ge 0$.  Tallying parameters in $\tilde{f}$, shows that $\tilde{f}$ has $3M+1$ non-zero parameters, $1$ hidden layer, and width $M$; as recorded in Table~\ref{tab:complexity_count__FullyExplicit} (and abbreviated in Table~\ref{tab:complexity_count}).
\end{proof}

\subsection{Detailed Euclidean Comparisons}
\label{s:Proofs__sss:DetailedEuclideanComparisons}

Together, the following Propositions imply Theorem~\ref{thrm:strictbetterthanEuc}.

\begin{proposition}[All Embeddings Implementable MLPs are Implementable by a Neural Snowflakes]
\label{prop:EucMatching}
For any $d,D,P\in \mathbb{N}_+$, $0\le s\le L$.  For any weighted finite graph $G=(E,V,W)$ with $V\subset \mathbb{R}^D$: if there is an MLP $\tilde{\mathcal{E}}:\mathbb{R}^D\rightarrow\mathbb{R}^d$ with $P$ non-zero parameters satisfying 
\[
        s\,d_G(v,u)
    \le 
        \|\tilde{\mathcal{E}}(v)-\tilde{\mathcal{E}}(u)\|
    \le
        L\,
        d_G(v,u)
\]
for each $v,u\in V$.  There is a pair of an MLP $\mathcal{E}:\mathbb{R}^D\rightarrow\mathbb{R}^d$ and a neural snowflake $f$ satisfying
\[
        s\,d_G(v,u)
    \le 
        \|\tilde{\mathcal{E}}(v)-\tilde{\mathcal{E}}(u)\|_f
    \le
        L\,
        d_G(v,u)
\]
for each $v,u\in V$.  Furthermore, the total number of non-zero parameters in $\mathcal{E}$ and $f$ is $P+4$.
\end{proposition}

\begin{proof}[{Proof of Proposition~\ref{prop:EucMatching}}]
Set $\mathcal{E}\eqdef \tilde{\mathcal{E}}$; in particular, $\mathcal{E}$ is defined by $P$ parameters.  

It remains to show that $f$ can implement the identity function.  
In the notation of~\eqref{eq:neural_snowflake}, set $I=1$, and consider the parameters 
\allowdisplaybreaks
\begin{align*}
    A^{(1)}=(1)
    ,\,
    B^{(1)}=(1)
    ,\,
    C^{(1)} = \begin{pmatrix}
        0\\
        1\\
        0
    \end{pmatrix}
    \alpha = 1
    ,\,
    \beta=0
.
\end{align*}
Then, $f$ defined by~\eqref{eq:neural_snowflake} satisfies $f(t)=t$.  Furthermore, $f$ has $1$ hidden layer, width $3$, and $4$ trainable parameters.  The conclusion now follows since $\tilde{\mathcal{E}}$ was assumed to implement an $(s,L)$-bi-Lipschitz embedding of $(V,d_G)$ into $(\mathbb{R}^d,\|\cdot\|)$.
\end{proof}

\begin{proposition}[Neural Snowflakes can Implement Isometric Embeddings which MLPs Cannot]
\label{prop:Strict_Improvement}
For any $d,D\in \mathbb{N}_+$ there exists a fully-connected weighted graph $G=(V,E,W)$ with $V\subset\mathbb{R}^D$ which:
\begin{enumerate}
    \item[(i)] Cannot be isometrically embedded into $(\mathbb{R}^n,\|\cdot\|)$ for any $n\le d$,
    \item[(ii)] There exists a neural snowflake $(\mathcal{E},f)$ with $\mathcal{E}:\mathbb{R}^D\rightarrow \mathbb{R}^d$ satisfying: for each $x,u\in V$
    \[
            \|\mathcal{E}(x)-\mathcal{E}(u)\|_f 
        = 
            d_G(x,u)
    .
    \]
\end{enumerate}
\end{proposition}

\begin{proof}[{Proof of Proposition~\ref{prop:Strict_Improvement}}]
Fix $d,D\in \mathbb{N}_+$, $\alpha=1/2$, and $p\eqdef \alpha$.  Then, \citep[Theorem 1.1]{LeDonneRajalaWalsberg_2018_AMS} implies that there exists some $N\in \mathbb{N}_+$ such that for any metric space $(V,d_G)$ with at-least $N$ points, the $1/2$-snowflake $(V,d_G^{1/2})$ cannot be isometrically embedded in $(\mathbb{R}^d,\|\cdot\|)$.  If $d>1$, suppose that for some $n< d$, $(V,d_G^{1/2})$ admitted an isometric embedding $\varphi_1:(V,d^{1/2})\rightarrow (\mathbb{R}^n,\|\cdot\|)$ then, since the map $\varphi_2:\mathbb{R}^n\rightarrow \mathbb{R}^d$ given by $z\mapsto (z_1,\dots,z_n,0,\dots,0)$, and since the composition of isometries is again an isometry, then $\varphi\eqdef \varphi_2\circ \varphi_1$ would define an isometry from $(V,d_G^{1/2})$ to $(\mathbb{R}^d,\|\cdot\|)$; which is a contradiction.  This, yields (i).

Let us now show (ii).  Set $p=2$ and let $V$ be any finite subset of $\mathbb{R}^D$ with $N$ points.  Consider the fully-connected graph $G=(V,E,W)$ with edge weights given by
\[
        W(x,u)
    \eqdef 
        \|x-u\|^{1/2}
    .
\]
By construction the graph geodesic distance $d_G$ on $G$ satisfies $d_G(x,u)=W(x,u)$ for every $x,u\in V$.  Furthermore, again by construction, for every $x,u\in V$ we have 
\[
    d_G(x,u) = W(x,u) = \|x-u\|^{1/2} = \|\mathcal{E}(x)-\mathcal{E}(u)\|_f
\]
where $f(t)=|t|^{1/2}$ and $\mathcal{E}(x)=x$.  Applying \citep[Lemma 20]{kratsiosembedding_2021}, we find that there exists an MLP with $\operatorname{ReLU}$ activation function $\Phi:\mathbb{R}^D\rightarrow\mathbb{R}^d$ with, depth, width, and number of non-zero parameters specified therein, satisfying $\Phi(x)=(x)$ for every $x\in V$.
\end{proof}

\subsection{Proof of Impossibility Results - Proposition~\ref{prop:Embedding_Impossible}}
\label{s:prop:Embedding_Impossible__PROOF}

We now prove Proposition~\ref{prop:Embedding_Impossible} by showing that the graph depicted in Figure~\ref{fig:badgraph} cannot be embedded isometrically into any Riemannian manifold, as we now show.

\begin{proof}[{Proof of Proposition~\ref{prop:Embedding_Impossible}}]
Fix $D\in \mathbb{N}_+$ and let $V\subseteq \mathbb{R}^D$ be any $5$-point subset; whose points we list by $V=\{A,B,C,D,E\}$.  Define the set of edges 
\[
        E
    \eqdef 
        \big\{
            \{A,E\}, \{A,D\}, \{E,B\}, \{B,D\}, \{D,C\}
        \big\}
\]
and consider the graph $\mathcal{G}\eqdef (V,E)$.  It is easy to see that $\mathcal{G}$ is connected; thus, the shortest path (geodesic) distance $d_{\mathcal{G}}$ on $V$ is well-defined.  
Furthermore, there are unique shortest paths joining $C$ to $A$ and $C$ to $B$; which we denote by $[C:A]$ and $[C:B]$ respectively; these are given by the ordered tuples (ordered pairs in this case)
\begin{equation}
\label{eq:unique_paths}
        [C:A]
    \eqdef 
        \big(
            \{C,D\},\{D,A\}
        \big)
    \mbox{ and }
        [C:B]
    \eqdef 
        \big(
            \{C,D\},\{D,B\}
        \big)
.
\end{equation}
We now argue by contradiction.  

Suppose that there exists a complete and connected smooth Riemannian manifold $(\mathcal{R},g)$ and an isometric embedding $\varphi:(V,d_{\mathcal{G}})\rightarrow (\mathcal{R},d_g)$; where, $d_g$ denotes the shortest path (geodesic) distance on $(\mathcal{R},g)$.  Since $(\mathcal{R},g)$ is complete (as a metric space) and connected, then the Hopf-Rinow Theorem \citep[Theorem 1.7.1]{Jost_2017_RGGA} implies that each pair of points in $\mathcal{R}$ can be joined by a distance minimizing geodesic, i.e.\ it is geodesically complete \citep[Definition 1.7.1]{Jost_2017_RGGA}.  Therefore, there exists a pair of geodesics $\gamma_{[C:A]}:[0,1]\rightarrow \mathcal{R}$ and $\gamma_{[C:B]}:[0,1]\rightarrow \mathcal{R}$ satisfying
\begin{align}
\label{eq:geodesic_definitions}
    \gamma_{[C:i]}(0) = \phi(C) 
    \mbox{ and }
    \gamma_{[C:i]}(1) = \phi(i) 
\end{align}
for $i\in \{A,B\}$.  
Since $\varphi$ is an isometric embedding and $\{C,D\}\in [C:A]\cap [C:B]$ then~\eqref{eq:unique_paths} implies that there is some $0<t_2<1$ for which
\begin{align}
\label{eq:impossiblesplitting_the_bad_point}
    \gamma_{[C:A]}(t_2) = \gamma_{[C:B]}(t_2) = \varphi(D)
.
\end{align}
Now,~\eqref{eq:geodesic_definitions} and the local uniqueness of geodesics in $(\mathcal{R},g)$ about any point, in particular about $\varphi(D)$ (see \citep[Theorem 1.4.2]{Jost_2017_RGGA}) imply that there is some $\varepsilon>0$ such that
\begin{align}
\label{eq:impossiblesplitting_coincidence}
    \gamma_{[C:A]}(s) = \gamma_{[C:B]}(s)
,
\end{align}
for all $t_2-\varepsilon \le s\le t_2 + \varepsilon$.  
Now since $A\neq B$ and since $\varphi$ is injective then $\varphi(A)\neq \varphi(B)$.  
However,~\eqref{eq:geodesic_definitions} and~\eqref{eq:impossiblesplitting_coincidence} cannot simultaneously hold for $t_2<s<t_2+\varepsilon$; thus we have a contradiction.  Consequentially, a pair $\big(\varphi,(\mathcal{R},d_{\mathcal{R}})\big)$ cannot exist.
\end{proof}

\section{Details on Examples}
\label{s:Further_Details}
This appendix contains further details, explaining derivations of particular examples within the paper's main text.  

\begin{exampledets}[Details for Example~\ref{ex:Family_of_Examples}]
\label{ex:Family_of_Examples__DETS}
This follows from \citep[Proposition 3.2]{GenShoenberg_2019_AA} and the Hausdorff-Bernstein-Widder theorem, see \citep[Theorem IV.12a]{Widder_LaplaceTransformBook1941}, together states that $f\circ u$ satisfies Assumption~\ref{assumptions:boundedmetric_MGF} if $f$ does and if $u:t\mapsto \frac{t}{a+t}$ is a \textit{Bernstein function}, i.e.\ a continuous function $f:[0,\infty)\rightarrow[0,\infty)$ which is smooth on $(0,\infty)$ and whose derivatives satisfy $(-1)^k\partial_t^k\,f(t)\le 0$ for all $t\ge 0$ and all $k\in \mathbb{N}_+$.  A long list of Bernstein functions which can be found in the several tables in \citep[Section 16.2]{ShillingSongZoran_BersteinFunctionsBook_2012deGruyter}.  For instance, the map $t\mapsto (1+\sum_{k=1}^K\,|t|^{r_k})^{-1}$ is one a Bernstein function, see \citep[Corollary 6.3]{ShillingSongZoran_BersteinFunctionsBook_2012deGruyter}, and $f(t)=|t|^{\beta}$ satisfies Assumption~\ref{assumptions:boundedmetric_MGF}.
\end{exampledets}

\begin{exampledets}[Details for Example~\ref{ex:inverse_multiqudratics}]
\label{ex:inverse_multiqudratics__DETS}
This holds, by arguing as in the previous example, since $t\mapsto e^{-b\,t}$ satisfies Assumption~\ref{assumptions:boundedmetric_RBF} and since $t\mapsto \frac{-(t-1)}{log(t)}$ is a Bernstein function, by \citep[Corollary 6.3]{ShillingSongZoran_BersteinFunctionsBook_2012deGruyter}.
\end{exampledets}

\section{Discrete Edge Sampling and the Discrete Differentiable Graph Module}
\label{s:Discrete Edge Sampling Algorithms}

Most latent graph inference models require generating a discrete graph based on a similarity measure between latent node representations. The discrete Differentiable Graph Module (dDGM)~\citep{Kazi_2022} has served as a source of inspiration for numerous studies in the field of latent graph inference~\citep{Manifold-dDGM,borde2023projections,battiloro2023latent}. It generates a sparse $k$-degree graph using the Gumbel Top-k trick~\citep{top-k}, a stochastic relaxation of the kNN rule, to sample edges from the probability matrix $\mathbf{P}^{(l)}(\mathbf{X}^{(l)};\mathbf{\Theta}^{(l)},T)$, where each entry corresponds to

\begin{equation}
    p_{ij}^{(l)}= \exp(-\varphi(\mathbf{\hat{x}}_{i},\mathbf{\hat{x}}_{j},T)).
    \label{pij}
\end{equation}

where $T$ is a learnable temperature parameter, $\hat{x}$ are latent node feature representation, and $\varphi$ is some similarity measure. In practice, the main similarity measure used in~\cite{Kazi_2022} was to compute the distance based on the features of two nodes in the graph embedding space, which was assumed to be Euclidean. Based on 

\begin{equation}
    \textrm{argsort}(\log(\mathbf{p}_{i}^{(l)})-\log(-\log(\mathbf{q})))
\end{equation}

where $\mathbf{q} \in \mathbb{R}^{N}$ is uniform i.i.d in the interval $[0,1]$, we can sample the edges

\begin{equation}
    \mathcal{E}^{(l)}(\mathbf{X}^{(l)};\mathbf{\Theta}^{(l)},T,k)=\{(i,j_{i,1}),(i,j_{i,2}),...,(i,j_{i,k}):i=1,...,N\},
\end{equation}

where $k$ is the number of sampled connections using the Gumbel Top-k trick. This sampling approach follows the categorical distribution $\frac{p_{ij}^{(l)}}{\Sigma_{r}p_{ir}^{(l)}}$ and $\mathcal{E}(\mathbf{X}^{(l)};\mathbf{\Theta}^{(l)},T,k)$ is represented by the unweighted adjacency matrix $\mathbf{A}^{(l)}(\mathbf{X}^{(l)};\mathbf{\Theta}^{(l)},T,k)$. Note that including noise in the edge sampling approach will result in the generation of some random edges in the latent graphs which can be understood as a form of regularization.

\section{Graph Learning Algorithms: Training and Backpropagation}
\label{s:Graph Learning Algorithms, Training and Backpropagation}

The optimization of the baseline node feature learning component in the architecture, that is, the standard GNN part relies on the loss of the downstream task. In particular, for classification, the cross-entropy loss is utilized. However, it is also necessary to update the parameters of graph learning modules such as the DGM~\citep{Kazi_2022}, the DCM~\citep{battiloro2023latent}, and the neural snowflake. To accomplish this, we implement a compound loss that provides incentives for edges contributing to accurate classification while penalizing edges that lead to misclassification. We introduce a reward function 

\begin{equation*}
    \delta\left(y_i,\hat{y}_i\right)=\mathds{E}(ac_i)-ac_i.
\end{equation*}

The aforementioned disparity is calculated as the difference between the mean accuracy of the $i$th sample and the present accuracy of the prediction. Here, $y_i$ and $\hat{y}_i$ represent the true and predicted labels, respectively, while $ac_i$ is assigned a value of 1 if $y_i=\hat{y}_i$, and 0 otherwise. The loss function for graph learning is formulated in terms of the reward function:

\begin{equation}
    L_{GL}=\sum_{i=1}^{N}\left(\delta\left(y_i,\hat{y}_i\right)\sum_{l=1}^{l=L}\sum_{j:(i,j)\in\hat{\varphi}^{(l)}}\log p_{ij}^{(l)}\right),
    \label{additional_loss}
\end{equation}

and it approximates the gradient of the expectation $\mathds{E}_{(\mathcal{G}^{(1)},...,\mathcal{G}^{(L)})\sim(\mathbf{P}^{(1)},..,\mathbf{P}^{(L)})}\sum_{i=1}^{N}\delta\left(y_i,\hat{y}_i\right).$ The expectation $\mathds{E}(ac_i)^{(t)}$ is calculated based on

\begin{equation}
    \mathds{E}(ac_i)^{(t)}=\beta\mathds{E}(ac_i)^{(t-1)}+(1-\beta)ac_i,
\end{equation}

with $\beta=0.9$ and $\mathds{E}(ac_i)^{(t=0)}=0.5$. For further details refer to~\cite{Kazi_2022} and~\cite{Manifold-dDGM}.

\section{Neural Snowflake for Latent Graph Inference Algorithms}
\label{appendix: Neural Snowflake Applied to Latent Graph Inference Algorithms}

This appendix includes Algorithm~\ref{alg:cap_forward},~\ref{alg:cap_dgm}, and~\ref{alg:cap_neural_snowflake}.  They summarize how learned latent graphs are incorporated into a standard GNN pipeline, how the dDGM samples graph edges, and how the neural snowflake architecture computes similarities between latent node representations, respectively. Superscripts are employed to denote layer-specific quantities, while subscripts are utilized for indices.

\begin{algorithm}
\caption{Node Level Prediction leveraging Inferred Latent Graph (Forward Pass)}\label{alg:cap_forward}
\begin{algorithmic}
\SetAlgoLined
\Require $\mathbf{X}, \mathbf{A}$ \Comment{Node Features and Adjacency Matrix} \\
\Return $\mathbf{Y}$ \Comment{Predicted Node Labels}

\DontPrintSemicolon
    \State $\mathbf{X}^{(0)} \gets \mathbf{X}$
    \State $\mathbf{\hat{A}} \gets \texttt{DGM}(\mathbf{X}^{(0)},\mathbf{A})$ \Comment{Refer to Algorithm~\ref{alg:cap_dgm}}
    \State \SetKwBlock{ForParallel}{For $l = 1$ to $L$}{end}
    \ForParallel{\State 
        \State $\mathbf{X}^{(l)} \gets \texttt{GNN}^{(l)}(\mathbf{X}^{(l-1)}, \mathbf{\hat{A}})$ \Comment{GNN diffusion layers}
    }
    \State $\mathbf{Y} \gets \texttt{MLP}(\mathbf{X}^{(L)})$ \Comment{Node level prediction}
\end{algorithmic}   
\end{algorithm}

Algorithm~\ref{alg:cap_dgm}, is a mild modification of the differentiable graph model of~\cite{Kazi_2022}.  Briefly, it allows the GNN to update its graph in a differentiable manner.

\begin{algorithm}
\caption{Discrete Differentiable Graph Module (modified based on~\cite{Kazi_2022})}\label{alg:cap_dgm}
\begin{algorithmic}
\SetAlgoLined
\Require $\mathbf{X}, \mathbf{A}$ \Comment{Node Features and Adjacency Matrix} \\
\Return $\mathbf{\hat{A}}$ \Comment{Latent Graph}

\DontPrintSemicolon
    \State $\mathbf{\hat{X}} \gets f(\mathbf{X},\mathbf{A})$ \Comment{Transform node features}
    \State $s_{ij} \gets \texttt{Neural Snowflake}(\mathbf{\hat{x}_i},\mathbf{\hat{x}_j})$ \Comment{Compute similarity measures. Refer to Algorithm~\ref{alg:cap_neural_snowflake}}
    \State $p_{ij} \gets g(s_{ij})$ \Comment{Compute edge sampling probabilities based on similarities}
    \State \SetKwBlock{ForParallel}{For $i = 1$ to $N$}{end}
    \ForParallel{\State 
        \State $\mathbf{q} \sim U(0,1)$ \Comment{Uniform i.i.d.}
        \State $\mathbf{j}_{\{k\}}=\texttt{argtopk}(\log\mathbf{p}_i-\log(-\log(\mathbf{q}_i)))$
        \State $\hat{a}_{ij}= \begin{cases}
        1 & \text{$j \in \mathbf{j}_{\{k\}}$}\\
        0 & \text{otherwise}\\
        \end{cases}$
    }
    \State $\mathbf{\hat{A}} \gets \hat{a}_{ij}$ \Comment{Discrete Latent Unweighted Graph Prediction}
\end{algorithmic}   
\end{algorithm}

Algorithm~\ref{alg:cap_neural_snowflake} describes how the neural snowflake processes node-level features to distances.

\begin{algorithm}
\caption{Neural Snowflake Processing}\label{alg:cap_neural_snowflake}
\begin{algorithmic}
\SetAlgoLined
\Require $\mathbf{\hat{x}_i},\mathbf{\hat{x}_j}$ \Comment{Two Node Features Vectors} \\
\Return $s_{ij}$ \Comment{Distance similarity measure}

\DontPrintSemicolon
    \State $t \gets ||\mathbf{\hat{x}_i}-\mathbf{\hat{x}_j}||_{2}$ \Comment{Euclidean Distance}
    \State $t^{(0)} \gets t$
    \State \SetKwBlock{ForParallel}{For $l = 1$ to $I$}{end}
    \ForParallel{\State 
        \State $\hat{t}^{(l-1)} \gets A^{(l)}t^{(l-1)}$ \Comment{Linear Projection}
        \State $\Sigma^{(l)} \gets \sigma^{(l)}(\hat{t}^{(l-1)})$ \Comment{Trainable Snowflake Activation (Equation~\ref{eq:activation})}
        \State $t^{(l)} \gets B^{(l)}\Sigma^{(l)}C^{(l)}$ \Comment{Linear Projections}
    }
    \State $s_{ij} \gets t_{I}^{1+|p|}$ \Comment{Quasi-metric}
\end{algorithmic}   
\end{algorithm}

\section{Computational Implementation Details}
\label{s:Computational Implementation}

In this appendix, we present additional information regarding the practical implementation of trainable snowflake activations and neural snowflakes, beyond the theoretical foundation discussed in the main text. We address certain instabilities encountered during the training process and propose methods to mitigate them. Our objective is to provide a deeper understanding of our findings, hoping that this will contribute to the advancement of neural snowflakes in future iterations.

\textbf{Hardware and Symbolic Matrices.}  In line with previous work, for most of the experiments, we utilized GPUs such as the NVIDIA Tesla T4 Tensor Core with 16 GB of GDDR6 memory, NVIDIA P100 with 16 GB of CoWoS HBM2 memory, or NVIDIA Tesla K80 with 24 GB of GDDR5 memory. These GPUs have limited memory capacities that are easily surpassed during backpropagation when dealing with datasets other than Cora and CiteSeer. One of the primary computational limitations of the Differentiable Graph Module and other latent graph inference techniques is the necessity to compute distances between latent representations for all nodes in order to generate the latent graph. While the discrete graph sampling method utilized by dDGM offers improved computational efficiency compared to its continuous counterpart, cDGM, due to the creation of sparse graphs that lighten the burden on convolutional operators, we encounter memory constraints when dealing with graph datasets containing a large number of nodes, on the order of $10^{4}$ nodes. To determine whether a connection should be established, we must calculate distances between all points starting from a pointcloud. However, this poses a challenge as the computational complexity scales quadratically with the number of nodes in the graph. Consequently, as the graph size increases, the computation quickly becomes intractable. To address the issue of potential memory overflows, we adopt Kernel Operations (KeOps)~\citep{keops}, as recommended by previous studies on latent graph inference. KeOps enables us to perform computations on large arrays by efficiently reducing them based on a mathematical formula.

\textbf{Trainable Snowflake Activation Preliminary Experiments: Stability and Initialization.} Although in~\eqref{eq:activation_new} the $\alpha$, $\beta$ and $\gamma$ parameters are introduced as trainable parameters for the sake of generality, we find that during backpropagation this exponential learnable terms can lead to instabilities, hence, we set them all to $\alpha=\beta=\gamma=1$. In future research it could be explored how to stabilize these and whether this additional flexibility proves advantageous empirically. We run some initial experiments to assess the stability of the trainable snowflake activation. In particular we work with the homophilic benchmarks Cora and Citeseer and we incorporate the snowflake activation to dDGM with Euclidean space and which feeds its infered latent graph to a Graph Convolutional Network of 3 layers with ELU activation functions. That is, the snowflake activation takes as input the euclidean distance between latent graph nodes computed by the dDGM. 

We observe that in this particular configuration the $p$ parameter that controls how much the quasi-metric deviates from being a metric can be a problematic parameter during training. Interestingly, this does not seem to be the case, when running synthetic experiments with a full neural snowflake (see Appendix \ref{s:Experimental Results Supplementary Material}). This could be attributed to the fact that in the literature in this setup the dDGM is trained with a learning rate of $10^{-2}$, which is too aggressive to update the $p$ parameter. For Cora the dDGM leveraging Euclidean space and using the original dataset graph as inductive bias achieves an accuracy of $82.40 \pm 3.22$ (mean $\pm$ standard deviation), using the off-the-shelf snowflake activation we obtain $40.33 \pm 11.62$ which presents a clear drop in performance. The observed large standard deviation indicates that the model frequently becomes trapped in local minima during optimization. This can be mitigated by either setting $p=0$ and learning a metric, or by using a different optimizer with a lower learning rate ($10^{-4}$, for example) to update the parameter $p$. This configurations lead to accuracies of $85.48 \pm 2.74$ and $86.11 \pm 3.72$ respectively for Cora, which clearly surpass the performance using Euclidean space. In the case of Citeseer using the original dDGM we get an accuracy of $73.40 \pm 1.64$, using a snowflake activation with $p=0$ we obtain $73.85 \pm 2.34$, and using a learnable $p$ with a learning rate of $10^{-4}$ we achieve $74.40 \pm 2.08$. Note that for the cases in which we learn $p$ with a different optimizer, $p$ is initialized at $p=10^{-8}$ to start from a metric and slowly learn a quasi-metric. These experiments show that the quasi-metric relaxation can provide the snowflake activation with additional representation capabilities and flexibility but should be used with care. For all the rest of experiments in this paper we learn a quasi-metric when implementing the snowflake activation but use a slower optimizer for the $p$ parameter. In these experiments the coefficients $C_1$, $C_2$, and $C_3$ were initialized to 1. We use the absolute value function during training to ensure they stay non-negative. Note that these experiments were conducted using the Gumbel Top-k trick edge sampling algorithm with $k=7$ and an embedding dimensionality of $4$. 

\textbf{Neural Snowflakes.} From a computational perspective, it is more reliable to assign fixed coefficients $a$ and $b$, instead of backpropagating through them. Specifically, we set $a = 1$ and $b = 1$ for all experiments. Matrices $A$, $B$, and $C$ weights are initialized by sampling from a uniform distribution ranging from 0 to 1 and normalized by the matrix dimensions. For example, in the case of $A$, the weights are sampled from a distribution between 0 and $1/(d_{A1}d_{A2})$, where $d_{A1}$ is the number of rows and $d_{A2}$ is the number of columns of the matrix. We experimentally observe that other initializations such as drawing the weights from a Gaussian or using Xavier initialization can lead to instabilities and exploding numbers in the forward pass. To guarantee non-negativity of all weights throughout training, we apply an absolute function activation to the weights. $p$ is initialized to $p=1e-8\approx 0$ , so that we start from a metric space and gradually learn a quasi-metric space. Furthermore, it should be noted that the learnable coefficient $p$ is exclusively applied in the last layer of the neural snowflake model. This design choice allows us to track the coefficient $C$, which represents the relaxation of the triangle inequality. In our synthetic experiments, we employ a readily available neural snowflake model, while for latent graph inference, we utilize a weighted skip connection. The neural snowflake model takes the Euclidean distance between latent representations as input. We observe that employing a skip connection and initiating training with an almost Euclidean metric proves beneficial, particularly during the early stages of training.

\section{Experimental Results Supplementary Material}
\label{s:Experimental Results Supplementary Material}

Within this appendix, we provide supplementary information regarding the experimental results discussed in the main text. This encompasses details about the train and test splits, the precise training configurations applied, as well as supplementary visual representations illustrating the evolution of the model training process, along with additional experiments and their corresponding results.

\textbf{Synthetic Experiments.} All models are trained using the Adam optimizer with a learning rate of $1\times 10^{-4}$, for 40 epochs and with a batch size of 1,000. After approximately 20 epochs, we notice a tendency for learning to reach a plateau, particularly when dealing with neural snowflakes. As mentioned in the main text, for our experimental analysis, we concentrate on fully connected graphs. In this setup, the node coordinates are sampled randomly from a multivariate Gaussian distribution within a 100-dimensional hypercube in Euclidean space, represented as $\mathbb{R}^{100}$. The graph weights are determined based on the metrics outlined in Table~\ref{table:synthetic experiments}. The training sets have 4,000,000 data points and the test sets 10,000. We observe very little discrepancy between the performance of the models for training and testing sets.

The MLP model, which works in isolation and aims at approximating the metrics using Euclidean space $\|\textrm{MLP}(\mathbf{x}) - \textrm{MLP}(\mathbf{y})\|$, has a total of 5422 parameters, and its composed of 10 linear layers with 20 hidden dimensions and ReLU activation functions. The neural snowflake learning the metric on $\mathbb{R}^2$: $f(\|\textrm{MLP}(\mathbf{x}) - \textrm{MLP}(\mathbf{y})\|)$, is composed of 2 layers with hidden dimension of 20. Note that in Section~\ref{ss:NeuralSnowflakes} we define $A^{(i)}$ as a $\tilde{d}_{i}\times d_{i-1}$ matrix, $B^{(i)}$ as a $d_{i}\times \tilde{d}_{i}$-matrix. However, in this experiments we set $\tilde{d}_{i}=d_{i}=20$. The MLP used alongside the neural snowflake to project the node features in $\mathbb{R}^{100}$ to $\mathbb{R}^2$ consists of 5 layers with hidden dimension 20 with a total of 3,322 model parameters and also uses ReLU activations. Lastly, for the third case in which the neural snowflake learns directly in $\mathbb{R}^{100}$: $f(\|\mathbf{x} - \mathbf{y}\|)$ we reuse the same neural snowflake architecture as before with a total of 847 learnable parameters. 

Additionally, we provide some plots of the training loss function evolution during learning for the synthetic graph embedding experiments in Figure~\ref{fig:synthetic_training_loss}. As we can observe from the plots the neural snowflakes learning in $\mathbb{R}^{100}$ converge faster, whereas neural snowflakes in $\mathbb{R}^{2}$ tend to get stuck in local minima at the beginning of training and eventually achieve comparable performance to their higher-dimensional counterparts. On the contrary, MLPs operating in $\mathbb{R}^{2}$ demonstrates significantly poorer performance, they achieve a higher loss with a higher variance during the training process. In the main text we provided results for the synthetic experiments in terms of the test set performance; we additionally provide results for the training set in Table~\ref{table:synthetic experiments train}.

\begin{figure}[hbtp!]
    \centering
    \begin{subfigure}[b]{0.32\textwidth}
        \includegraphics[width=\textwidth]{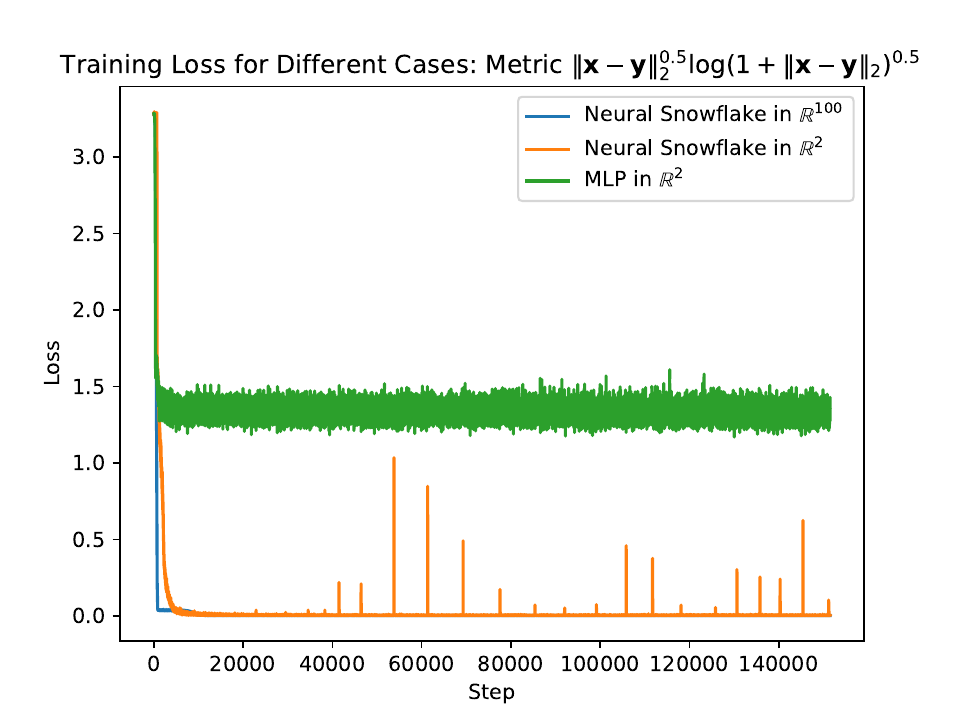}
        \caption{}
        \label{fig:subfig1}
    \end{subfigure}
    \hfill
    \begin{subfigure}[b]{0.32\textwidth}
        \includegraphics[width=\textwidth]{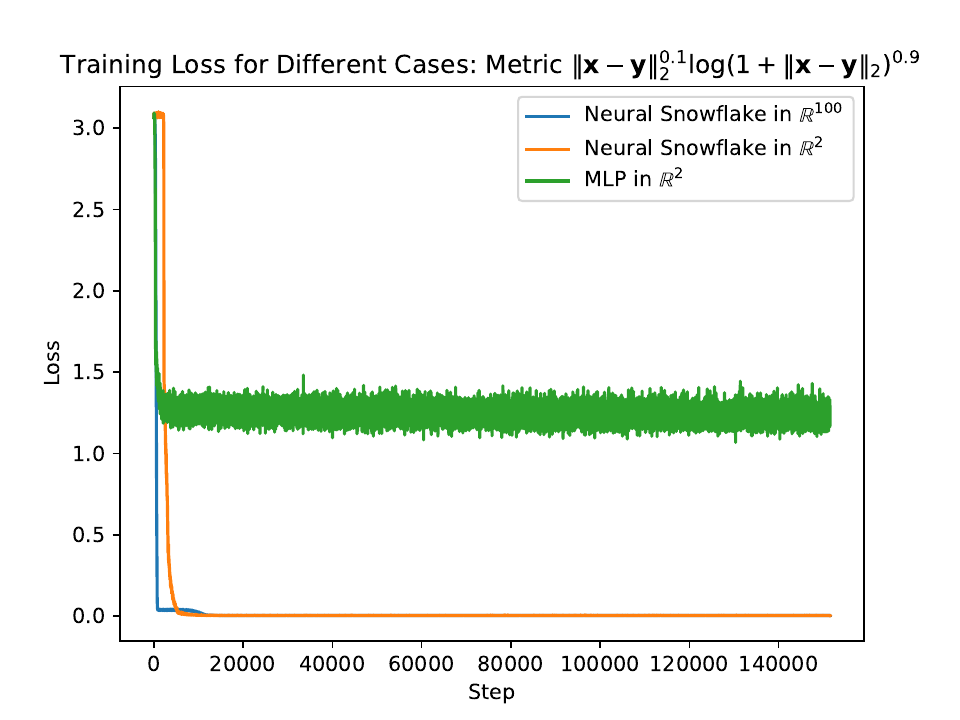}
        \caption{}
        \label{fig:subfig2}
    \end{subfigure}
    \hfill
    \begin{subfigure}[b]{0.32\textwidth}
        \includegraphics[width=\textwidth]{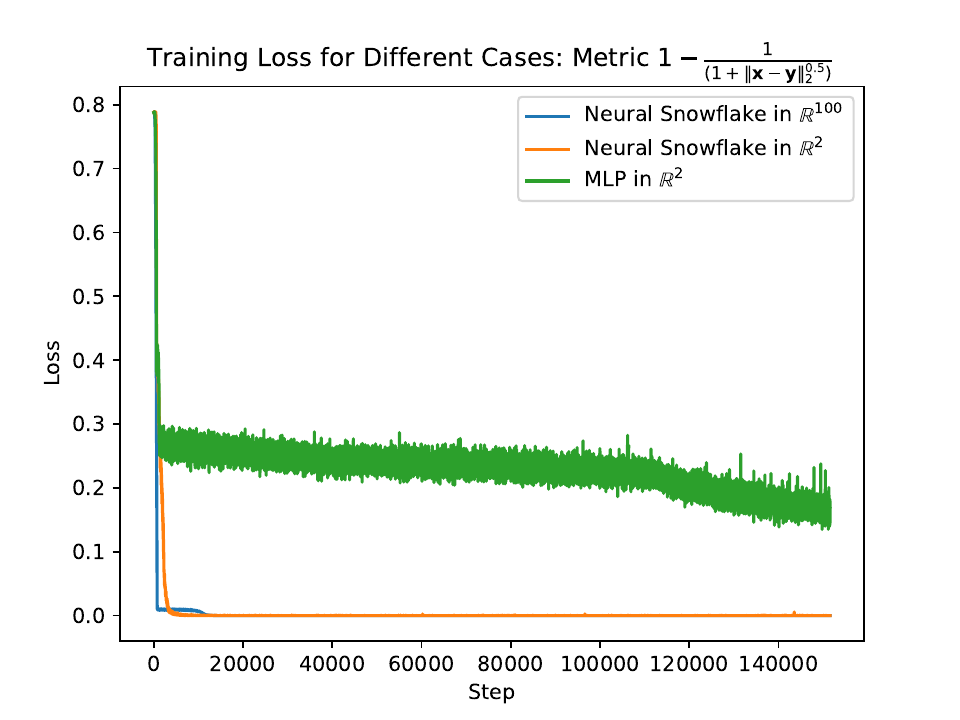}
        \caption{}
        \label{fig:subfig3}
    \end{subfigure}
    \vskip\baselineskip
    \begin{subfigure}[b]{0.32\textwidth}
        \includegraphics[width=\textwidth]{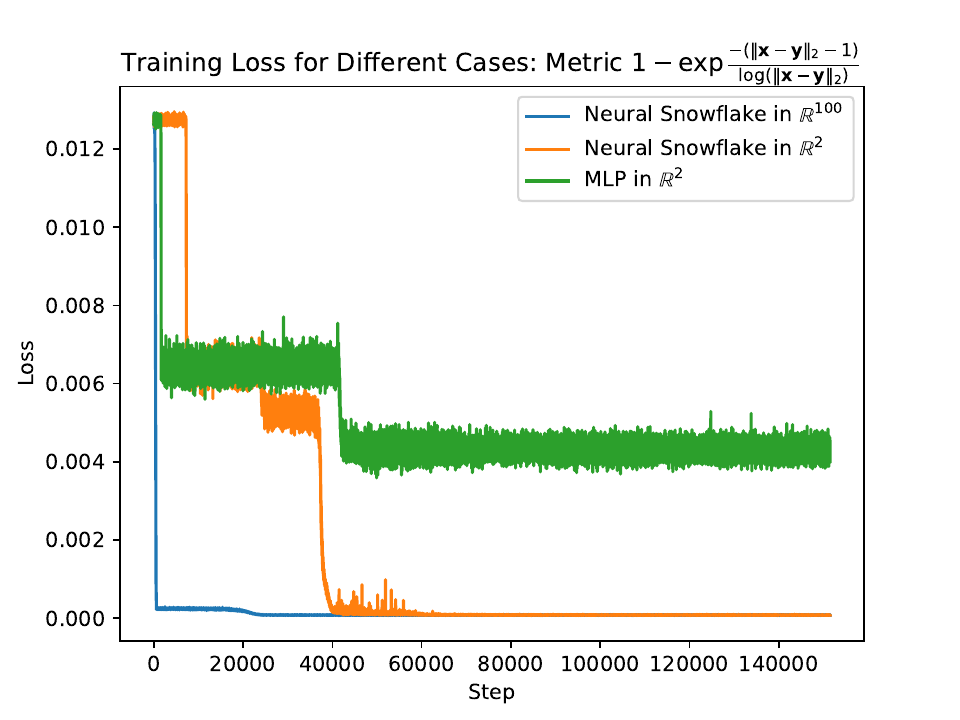}
        \caption{}
        \label{fig:subfig4}
    \end{subfigure}
    \hfill
    \begin{subfigure}[b]{0.32\textwidth}
        \includegraphics[width=\textwidth]{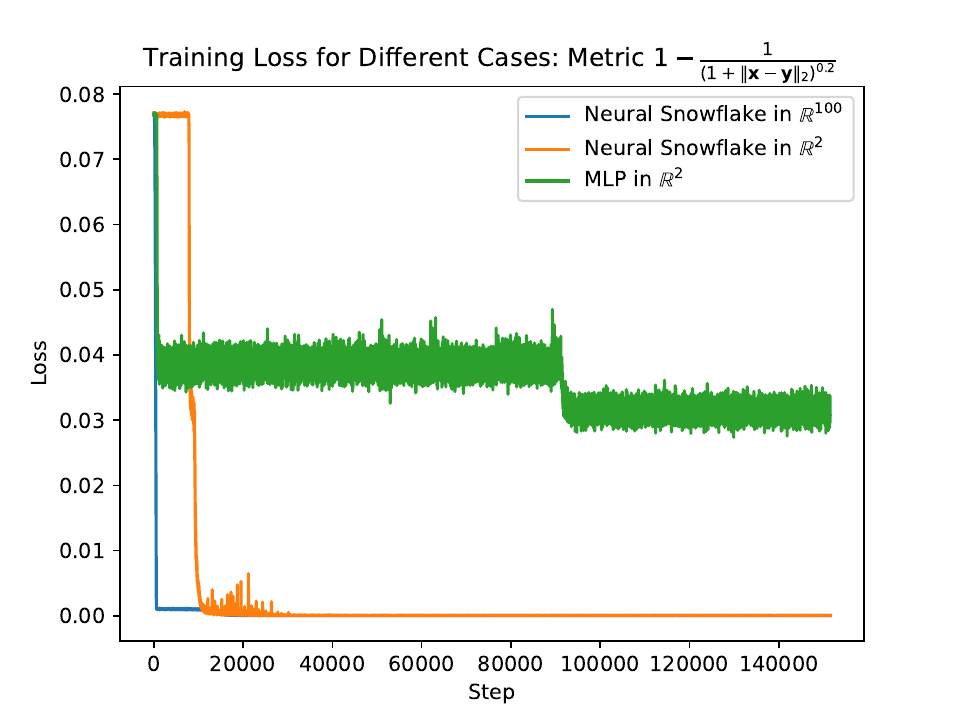}
        \caption{}
        \label{fig:subfig5}
    \end{subfigure}
    \hfill
    \begin{subfigure}[b]{0.32\textwidth}
        \includegraphics[width=\textwidth]{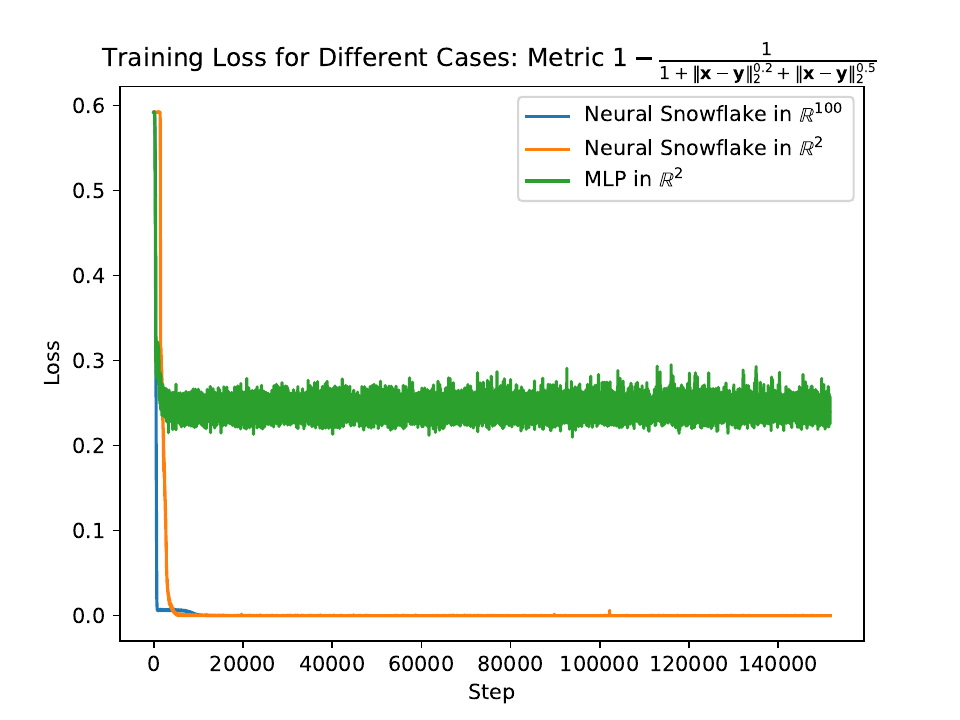}
        \caption{}
        \label{fig:subfig6}
    \end{subfigure}
    \caption{Training losses for synthetic graph embedding experiments. We compare using Euclidean space for encoding the weighted graphs to using snowflake quasi-metric spaces.}
    \label{fig:synthetic_training_loss}
\end{figure}

\begin{table*}[htbp!]
\small
\caption{Results for synthetic graph embedding experiments, mean square error for training set. The Neural Snowflake models are able to learn the metric better with substantially lesser number of model parameters.} 
\centering
\scalebox{0.65}{
\begin{tabular}{lccc}
    \toprule
      &\textbf{MLP} & \textbf{Neural Snowflake (+ MLP)}  &   \textbf{Neural Snowflake}\\ 
    No. Parameters & 5422  & 4169  &  847 \\ 
    Embedding space, $\mathbb{R}^n$ & 2 &  2 &  100 \\\midrule
     Metric &  \multicolumn{3}{c}{Mean Square Error}  \\ \midrule
     $\|\mathbf{x} - \mathbf{y}\|^{0.5}\log(1+\|\mathbf{x} - \mathbf{y}\|)^{0.5}$ &  1.3299 &  0.0037 &  $\mathbf{0.0029}$ \\\midrule
     $\|\mathbf{x} - \mathbf{y}\|^{0.1}\log(1+\|\mathbf{x} - \mathbf{y}\|)^{0.9}$ &  1.2242 &  0.0032 &  $\mathbf{0.0031}$\\\midrule

     $1- \frac{1}{(1+\|\mathbf{x} - \mathbf{y}\|^{0.5})}$  & 0.0777  & $\mathbf{0.00004}$   & $\mathbf{0.00004}$
     \\\midrule

     $1 - \exp{\frac{-(\|\mathbf{x} - \mathbf{y}\|-1)}{\log(\|\mathbf{x} - \mathbf{y}\|)}}$  & 0.1649 & 0.00009 & $\mathbf{0.00008}$
     \\\midrule

     $1- \frac{1}{(1+\|\mathbf{x} - \mathbf{y}\|)^{0.2}}$  &  0.0314 &  $\mathbf{0.00005}$  & $\mathbf{0.00005}$\\\midrule

     $1- \frac{1}{1+\|\mathbf{x} - \mathbf{y}\|^{0.2}+\|\mathbf{x} - \mathbf{y}\|^{0.5}}$  & 0.2420 & $\mathbf{0.00002}$ & $\mathbf{0.00002}$\\

    \bottomrule
\end{tabular}
}
\label{table:synthetic experiments train}
\end{table*}

\textbf{Snowflake activation.} Before working with neural snowflakes, we evaluate the performance of snowflake activations. This is a straightforward way of augmenting existing latent graph inference algorithms with additional representation power. In appendix~\ref{s:Computational Implementation}, we have already covered the details regarding the activation function's computational implementation. We intend to adhere to the specifications outlined in that section. We start by running some preliminary studies in which we compare the performance of a GCN equipped with a dDGM module using different latent embedding metric spaces. To ensure a fair comparison in terms of latent space geometry, we have set the dimensionality of the latent embedding space to 4. By keeping the dimensionality fixed, we ensure that we only modify the manifold used for embedding the representations, without altering the dimensionality of the embedding space itself. We start by evaluating the different algorithms on benchmark homophilic graph datasets such as Cora, CiteSeer, CS, and Physics. We start by comparing our model to previous metric spaces used in the literature such as Euclidean space~\citep{Kazi_2022}. In Table~\ref{tab:homophilic_latent_space_dim_4} we can observe that the snowflake activation helps achieve higher accuracies, leading to improvements between $1\%-5\%$.

\begin{table}[H]
    \centering
    \caption{Latent graph inference results across a variety of benchmark homophilic graph datasets. For all experiment the latent space dimensionality used to infer latent graphs is fixed to 4, and the Gumbel top-k trick is used with $k = 7$.}
    \label{tab:homophilic_latent_space_dim_4}
    \scalebox{0.65}{
    \begin{tabular}{lccccccc}
    \toprule 
         &
         & &
         \textbf{Cora} &  
         \textbf{CiteSeer} & 
         \textbf{CS} &
         \textbf{Physics}
         
         \\ 
         \midrule
         
         Model  & Metric Space & Input Graph &
         \multicolumn{4}{c}{Accuracy $(\%)$ $\pm$ Standard Deviation}
         
         \\

        \midrule

         DGM  & Euclidean & Yes &
         $ 82.40 {\scriptstyle \pm 3.22}$ &
         $ 73.40 {\scriptstyle \pm 1.64}$ &
         $ 85.45 {\scriptstyle \pm 2.23}$ &  
         $ 95.91  {\scriptstyle \pm 0.51 }$\\ 

         DGM  & Snowflake & Yes &
         $ 86.11 {\scriptstyle \pm 3.72}$ &
         $ 74.40 {\scriptstyle \pm 2.08}$ &
         $ 89.54 {\scriptstyle \pm 2.48}$ &  
         $   96.08 {\scriptstyle \pm  0.46}$\\ \midrule

         DGM  & Euclidean & No &
         $ 62.03 {\scriptstyle \pm 6.20}$ &
         $ 65.15 {\scriptstyle \pm 4.84}$ &
         $ 84.37 {\scriptstyle \pm 1.20}$ &  
         $ 95.11  {\scriptstyle \pm 0.33 }$\\ 

         DGM  & Snowflake & No &
         $ 67.33 {\scriptstyle \pm 4.10}$ &
         $ 64.63 {\scriptstyle \pm 9.24}$ &
         $ 87.63 {\scriptstyle \pm 5.25}$ &  
         $ 95.32  {\scriptstyle \pm  0.45}$\\

         \midrule

         MLP & Euclidean & No &
         $ 58.92 {\scriptstyle \pm 3.28}$ &
         $ 59.48 {\scriptstyle \pm 2.14}$ &
         $  87.80 {\scriptstyle \pm 1.54 }$ &
         $ 94.91 {\scriptstyle \pm 0.30}$ &  
         \\

         \bottomrule
         
    \end{tabular}
    }
\end{table}

Next, we increase the dimensionality of the latent space from 4 to 8 and evaluate whether this trend still persists. In Table~\ref{tab:homophilic_latent_space_dim_8}, we can observe that as we increase the dimensionality of the latent space used for inferring the latent graph the performance increases when using both Euclidean or snowflake quasi-metric spaces. We can also see that the difference between the two becomes less significant, except for the CS dataset in which the snowflake activation still performs substantially better.

\begin{table}[H]
    \centering
    \caption{Latent graph inference results with latent space dimensionality fixed to 8, and the Gumbel top-k trick is used with $k = 7$.}
    \scalebox{0.65}{
    \begin{tabular}{lccccccc}
    \toprule 
         &
         & &
         \textbf{Cora} &  
         \textbf{CiteSeer} & 
         \textbf{CS} &
         \textbf{Physics}
         
         \\ 
         \midrule
         
         Model  & Metric Space & Input Graph &
         \multicolumn{4}{c}{Accuracy $(\%)$ $\pm$ Standard Deviation}
         
         \\

        \midrule

         DGM  & Euclidean & Yes &
         $ 85.77 {\scriptstyle \pm 3.64}$ &
         $ 73.67 {\scriptstyle \pm 2.30}$ &
         $ 90.50 {\scriptstyle \pm 1.89}$ &  
         $ 96.08 {\scriptstyle \pm 0.41}$\\ 

         DGM  & Snowflake & Yes &
         $ 85.41 {\scriptstyle \pm 3.70}$ &
         $ 74.19 {\scriptstyle \pm 2.08}$ &
         $ 92.98 {\scriptstyle \pm 0.66}$ &  
         $ 96.15 {\scriptstyle \pm 0.54}$\\ \midrule

         DGM  & Euclidean & No &
         $ 68.37 {\scriptstyle \pm 5.39}$ &
         $ 68.10 {\scriptstyle \pm 2.80}$ &
         $ 88.17 {\scriptstyle \pm 2.64}$ &  
         $ 95.27 {\scriptstyle \pm 0.41}$\\ 

         DGM  & Snowflake & No &
         $ 69.51 {\scriptstyle \pm 4.42}$ &
         $  66.86 {\scriptstyle \pm 2.82}$ &
         $ 88.67 {\scriptstyle \pm 3.21}$ &  
         $  95.36 {\scriptstyle \pm 0.23}$\\

         \midrule

         MLP & Euclidean & No &
         $ 58.92 {\scriptstyle \pm 3.28}$ &
         $ 59.48 {\scriptstyle \pm 2.14}$ &
         $  87.80 {\scriptstyle \pm 1.54 }$ &
         $ 94.91 {\scriptstyle \pm 0.30}$ &  
         \\

         \bottomrule
         
    \end{tabular}
    }
    
    \label{tab:homophilic_latent_space_dim_8}
\end{table}

In the specific scenario presented in Table~\ref{tab:homophilic_latent_space_dim_2}, we observe a contrasting effect. As the latent space dimension is reduced to only 2, the performance of both models deteriorates. Nonetheless, it is notable that the use of the snowflake metric space demonstrates greater resilience compared to its Euclidean counterpart. In fact, we can observe performance discrepancies of up to $20\%$ between the two. This is in line with previous synthetic experiments, and demonstrates that snowflake quasi-metric spaces are more efficient at compressing the same information in low-dimensional spaces.

\begin{table}[H]
    \centering
    \caption{Latent graph inference results with latent space dimensionality fixed to 2, and the Gumbel top-k trick is used with $k = 7$.}
    \scalebox{0.65}{
    \begin{tabular}{lccccccc}
    \toprule 
         &
         & &
         \textbf{Cora} &  
         \textbf{CiteSeer} & 
         \textbf{CS} &
         \textbf{Physics}
         
         \\ 
         \midrule
         
         Model  & Metric Space & Input Graph &
         \multicolumn{4}{c}{Accuracy $(\%)$ $\pm$ Standard Deviation}
         
         \\

        \midrule

         DGM  & Euclidean & Yes &
         $ 59.25 {\scriptstyle \pm 14.60}$ &
         $ 70.00 {\scriptstyle \pm 2.15}$ &
         $ 62.15 {\scriptstyle \pm 2.92}$ &  
         $ 92.01 {\scriptstyle \pm 2.74}$\\ 

         DGM  & Snowflake & Yes &
         $ 79.44 {\scriptstyle \pm 6.50}$ &
         $ 69.51 {\scriptstyle \pm 3.95}$ &
         $ 79.75 {\scriptstyle \pm 2.57}$ &  
         $ 94.29 {\scriptstyle \pm 2.91}$\\ \midrule

         DGM  & Euclidean & No &
         $ 42.40 {\scriptstyle \pm 8.20}$ &
         $ 60.93 {\scriptstyle \pm 4.07}$ &
         $ 69.93 {\scriptstyle \pm 2.17}$ &  
         $ 86.05 {\scriptstyle \pm 2.77}$\\ 

         DGM  & Snowflake & No &
         $ 64.18 {\scriptstyle \pm 3.46}$ &
         $ 64.61 {\scriptstyle \pm 6.14}$ &
         $ 81.70 {\scriptstyle \pm 5.35}$ &  
         $ 93.45 {\scriptstyle \pm 2.93}$\\

         \midrule

         MLP & Euclidean & No &
         $ 58.92 {\scriptstyle \pm 3.28}$ &
         $ 59.48 {\scriptstyle \pm 2.14}$ &
         $  87.80 {\scriptstyle \pm 1.54 }$ &
         $ 94.91 {\scriptstyle \pm 0.30}$ &  
         \\

         \bottomrule
         
    \end{tabular}
    }
    
    \label{tab:homophilic_latent_space_dim_2}
\end{table}

\end{document}